\documentclass{article} 

\usepackage{arxiv,times}

\usepackage{amsmath,amsfonts,bm}

\def\eqref#1{equation~\ref{#1}}

\def\1{\bm{1}}

\DeclareMathAlphabet{\mathsfit}{\encodingdefault}{\sfdefault}{m}{sl}
\SetMathAlphabet{\mathsfit}{bold}{\encodingdefault}{\sfdefault}{bx}{n}

\newcommand{\R}{\mathbb{R}}

\usepackage{graphicx}
\usepackage{caption}
\usepackage{subcaption}
\usepackage{hyperref}
\RequirePackage{algorithm}
\RequirePackage{algorithmic}
\usepackage{amsthm}
\usepackage{amsmath}
\usepackage{amssymb}
\newtheorem{theorem}{Theorem}[section]
\newtheorem{lemma}{Lemma}[section]
\newtheorem{corollary}{Corollary}[section]
\newtheorem{definition}{Definition}[section]
\usepackage{tikz}
\usetikzlibrary{shapes}
\usepackage{xcolor}

\newcommand{\tr}{\operatorname{trace}}
\DeclareMathSymbol{\shortminus}{\mathbin}{AMSa}{"39}
\newcommand{\bGaminds}[1]{\textcolor{blue}{\Gamma_{#1}}}
\newcommand{\bGam}{\textcolor{blue}{\Gamma}}
\newcommand{\ppart}[1]{\textcolor{violet}{\max\Bigl(}{#1}\textcolor{violet}{,0\Bigr)}}
\usepackage{bbm} 
\newcommand{\ind}[1]{\mathbbm{1}_{\{{#1}\}}}  
\usepackage{mathtools}
\newcommand{\defeq}{\vcentcolon=}
\newcommand{\eqdef}{=\vcentcolon}  
\newcommand{\PD}[1]{\mathbf{S}^{#1}_{++}}
\newcommand{\PSD}[1]{\mathbf{S}^{#1}_{+}}
\usepackage{xfrac}
\newcommand{\mhalf}{{\shortminus \sfrac{1}{2}}}
\newcommand{\half}{{\sfrac{1}{2}}}  
\usepackage{thmtools}
\usepackage{thm-restate}
\usepackage{cleveref}

\title{A Generalized EigenGame With Extensions to Deep Multiview Representation Learning}

\author{James Chapman \& Ana Lawry Aguila\\
Department of Computer Science\\
University College London\\
90 High Holborn \\
\texttt{\{james.chapman.19,ana.aguila.19\}@ucl.ac.uk} \\
\And
Lennie Wells\\
Faculty of Mathematics\\
University of Cambridge\\
Wilberforce Road\\
\texttt{ww347@cam.ac.uk}
}

\iclrfinalcopy
\begin{document}

\maketitle

\begin{abstract}
Generalized Eigenvalue Problems (GEPs) encompass a range of interesting dimensionality reduction methods. Development of efficient stochastic approaches to these problems would allow them to scale to larger datasets. Canonical Correlation Analysis (CCA) is one example of a GEP for dimensionality reduction which has found extensive use in problems with two or more views of the data. Deep learning extensions of CCA require large mini-batch sizes, and therefore large memory consumption, in the stochastic setting to achieve good performance and this has limited its application in practice. Inspired by the Generalized Hebbian Algorithm, we develop an approach to solving stochastic GEPs in which all constraints are softly enforced by Lagrange multipliers. Then by considering the integral of this Lagrangian function, its pseudo-utility, and inspired by recent formulations of Principal Components Analysis and GEPs as games with differentiable utilities, we develop a game-theory inspired approach to solving GEPs. We show that our approaches share much of the theoretical grounding of the previous Hebbian and game theoretic approaches for the linear case but our method permits extension to general function approximators like neural networks for certain GEPs for dimensionality reduction including CCA which means our method can be used for deep multiview representation learning. We demonstrate the effectiveness of our method for solving GEPs in the stochastic setting using canonical multiview datasets and demonstrate state-of-the-art performance for optimizing Deep CCA.
\end{abstract}

\section{Introduction}

A Generalised Eigenvalue Problem (GEP) is defined by two symmetric\footnote{or, more generally, Hermitian \cite{stewart_matrix_1990}} matrices $A,B\in \mathbb{R}^{d\times d}$. They are usually characterised by the set of solutions to the equation 
\begin{align}\label{eq:igep}
Aw=\lambda Bw
\end{align}
with $\lambda \in \R, w \in \R^d$, called (generalised) eigenvalue and (generalised) eigenvector respectively.
Note that by taking $B=I$ we recover the standard eigenvalue problem. We shall only be concerned with the case where $B$ is positive definite to avoid degeneracy; in this case one can find a basis of eigenvectors spanning $\R^d$. Without loss of generality, take $w_1,\dots,w_d$ such a basis of eigenvectors, with decreasing corresponding eigenvalues $\lambda_1\geq \dots \geq \lambda_d$. 
The following variational characterisation \cite{stewart_matrix_1990} provides a useful alternative, iterative definition: $w_k$ solves
\begin{align}\label{eq:consgep}
\max_{w\in\R^d}w^{\top}Aw \quad\text{ subject to }w^{\top}Bw=1, \: w^{\top}Bw_j = 0 \text{ for } j=1,\dots,k \shortminus 1.
\end{align}
There is also a simpler (non-iterative) variational characterisation for the top-$k$ subspace (that spanned by $\{w_1,\dots,w_k\})$, namely 
\begin{equation}\label{eq:subspace-gep}
    \max_{W \in \R^{d\times k}} \tr(W^{\top} A W) \quad \text{ subject to }\:  W^{\top} B W = I_k
\end{equation}
again see \cite{stewart_matrix_1990}\footnote{We need a stronger treatment of this result so give a self contained statement and proof in Proposition \ref{prop:constr-charac}, with associated discussion.}; the drawback of this characterisation is it only recovers the subspace and not the individual eigenvectors. We shall see that these two different characterisations lead to different algorithms for the GEP.

Many classical dimensionality reduction methods can be viewed as GEPs including but not limited to Principal Components Analysis \citep{hotelling1933analysis}, Partial Least Squares \cite{haenlein2004beginner}, Fisher Discriminant Analysis \cite{mika1999fisher}, and Canonical Correlation Analysis (CCA) \citep{hotelling1992relations}. 

Each of the problems above is defined at a population level, using population values of the matrices $A,B$, usually functionals of some appropriate covariance matrices. The practical challenge is the sample version: to estimate the population GEP where we only have estimates of $A,B$ through some finite number of samples $(z_n)_{n=1}^N$; classically, one just solves the GEP with $A,B$ estimated by plugging in the relevant sample covariance matrices.

However for very large datasets, the dimensionality of the associated GEPs makes it memory and compute intensive to compute solutions using existing full-batch algorithms; these are usually variants of the singular value decomposition where successive eigenvalue eigenvector pairs are calculated sequentially by deflation \cite{mackey2008deflation} and so cannot exploit parallelism over the eigenvectors. 

This work was motivated in particular by CCA, a classical method for learning representations of data with two or more distinct views: a problem known as multiview (representation) learning. Multiview learning methods are useful for learning representations of data with multiple sets of features, or `views'. CCA identifies projections or subspaces in at least two different views that are highly correlated and can be used to generate robust low-dimensional representations for a downstream prediction task, to discover relationships between views, or to generate representations of a view that is missing at test time. CCA has been widely applied across a range of fields such as Neuroimaging \citep{Krishnan2011}, Finance \citep{cassel2000measurement}, and Imaging Genetics \citep{Hansen2021}.

Deep learning functional forms are often extremely effective for modelling extremely large datasets as they have more expressivity than linear models and scale better than kernel methods. While PCA has found a natural stochastic non-linear extension in the popular autoencoder architecture \citep{kramer1991nonlinear}, applications of Deep CCA \citep{andrew2013deep} have been more limited because estimation of the constraints in the problem outside the full batch setting are more challenging to optimize. In particular, DCCA performs badly when its objective is maximized using stochastic mini-batches. This is unfortunate as DCCA would appear to be well suited to a number of multiview machine learning applications as a number of successful deep multiview machine learning \citep{Suzuki2022} and certain self-supervised learning approaches \citep{zbontar2021barlow} are designed around similar principals to DCCA; to maximize the consensus between non-linear models of different views \cite{nguyen2020multiview}.

Recently, a number of algorithms have been proposed to approximate GEPs \cite{arora2012stochastic}, and CCA specifically \cite{bhatia2018gen}, in the `stochastic' or `data-streaming' setting; these can have big computational savings. Typically, the computational complexity of classical GEP algorithms is $\mathcal{O}\left((N + k)d^2\right)$; by exploiting parallelism (both between eigenvectors and between samples in a mini-batch), we can reduce this down to $\mathcal{O}\left(d k \right)$ \citep{arora2016stochastic}. Stochastic algorithms also introduce a form of regularisation which can be very helpful in these high-dimensional settings.

A key motivation for us was a recent line of work reformulating top-k eigenvalue problems as games \citep{gemp20,gemp2021}, later extended to GEPs in \cite{gemp2022generalized}. We shall refer to these ideas as the `Eigengame framework'. Unfortunately, their GEP extension is very complicated, with 3 different hyperparameters; this complication is needed because they constrain their estimates to lie on the unit sphere, which is a natural geometry for the usual eigenvalue problem but not natural for the GEP. By replacing this unit sphere constraint with a Lagrange multiplier penalty, we obtain a much simpler method (GHA-GEP) with only a single hyperparameter; this is a big practical improvement because the convergence of the algorithms is mostly sensitive to step-size (learning rate) parameter \cite{li2021nonconvex}, and it allows a practitioner to explore many more learning rates for the same computational budget. We also propose a second class of method ($\delta$-EigenGame) defined via optimising explicit utility functions, rather than being defined via updates, which enjoys the same practical advantages and similar performance.
These utilities give unconstrained variational forms for GEPs that we have not seen elsewhere in the literature and may be of independent interest; their key practical advantage is that they only contain linear factors of $A,B$ so we can easily obtain unbiased updates for gradients.
The other key advantage of these utility-based methods is that they can easily be extended to use deep learning to solve problems motivated by GEPs. In particular we propose a simple but powerful method for the Deep CCA problem.

\subsection{Notation}

We have collected here some notational conventions which we think may provide a helpful reference for the reader. We shall always have $A,B \in \R^{d\times d}$. We denote (estimates to or dummy variables for) the $i^\text{th}$ generalised eigenvectors by $w_i$; and denote CCA directions $u_i \in \R^p,v_i \in \R^q$. The number of directions we want to estimate will be $k$. For stochastic algorithms, we denote batch-size by $b$. We use $\langle \cdot, \cdot \rangle$ for inner products; implicitly we always take Euclidean inner product over vectors and Frobenius or `trace' inner product for matrices.

\section{A Constraint-Free Algorithm for GEPs}

Our first proposed method solves the general form of the generalized eigenvalue problem in equation (\ref{eq:consgep}) for the top-k eigenvalues and their associated eigenvectors in parallel. We are thus interested in both the top-k subspace problem and the top-k eigenvectors themselves. Our method extends the Generalized Hebbian Algorithm to GEPs, and we thus refer to it as GHA-GEP.

In the full-batch version of our algorithm when $A$ is known to be positive semidefinite, each eigenvector estimate has updates with the form

\begin{align}
\Delta_{i}^{\text{GHA-GEP-PSD}}
&=
\overbrace{A \hat{w}_{i}}^{\text{Reward}} - \overbrace{\sum_{j \leq i}B\hat{w}_{j}\textcolor{blue}{\left(\hat{w}_{j}^{\top} A \hat{w}_{i}\right)}}^{\text{Penalty}} 
&&=
\overbrace{A \hat{w}_{i}}^{\text{Reward}} - \overbrace{B\hat{w}_{i}\textcolor{blue}{\left(\hat{w}_{i}^{\top} A \hat{w}_{i}\right)}}^{\text{Variance Penalty}} - \overbrace{\sum_{j < i}B\hat{w}_{j}\textcolor{blue}{\left(\hat{w}_{j}^{\top} A \hat{w}_{i}\right)}}^{\text{Orthogonality Penalty}} & \nonumber\\ 
&=A \hat{w}_{i} - \sum_{j \leq i}B\hat{w}_{j}\textcolor{blue}{\Gamma_{ij}}
&&=
A \hat{w}_{i} - B\hat{w}_{i}\bGaminds{ii} - \sum_{j < i}B\hat{w}_{j}\bGaminds{ij}&
\label{eq:psdupdate}
\end{align}

where $\hat{w_j}$ is our estimator to the eigenvector associated with the $j^{\text{th}}$ largest eigenvalue and in the stochastic setting, we can replace $A$ and $B$ with their unbiased estimates $\hat{A}$ and $\hat{B}$. 
We will use the notation $\textcolor{blue}{\Gamma_{ij}=\left(\hat{w}_{j}^{\top} A \hat{w}_{i}\right)}$ to facilitate comparison with previous work in Appendix \ref{sec:previousworkcomparison}.  $\textcolor{blue}{\Gamma_{ij}}$ has a natural interpretation as Lagrange multiplier for the constraint $w_i^\top B w_j = 0$; indeed, \citet{chen2019constrained} prove that $\left(\hat{w}_{j}^{\top} A \hat{w}_{i}\right)$ is the optimal value of the corresponding Lagrange multiplier for their GEP formulation; we summarise this derivation in Appendix \ref{sec:chen} for ease of reference.
We also label the terms as rewards and penalties to facilitate discussion with respect to the EigenGame framework in Appendix \ref{sec:mueigengame} and recent work in self-supervised learning in Appendix \ref{sec:ssl}.

However, when $A$ has negative eigenvalues, the iteration defined in (\ref{eq:psdupdate}) can `blow-up' from certain initial values. We therefore propose the following modification:
\begin{equation}\label{eq:ourupdate}
        \Delta_{i}^{\text{GHA-GEP}}=
        A \hat{w}_{i} - B\hat{w}_{i}\,\textcolor{violet}{\max(}\bGaminds{ii}\textcolor{violet}{,0)} - \sum_{j < i}B\hat{w}_{j}\bGaminds{ij}
    \end{equation}
Note that this reduces to (\ref{eq:psdupdate}) when $A$ is positive semi-definite. The following proposition, proved in Appendix \ref{app:gha-gep}, justifies this choice of update.

\begin{restatable}[Unique stable stationary point]{proprep}{gepghastationary}
\label{prop:gep-gha-stat}
Given exact parents and assuming the top-k generalized eigenvalues
of $A$ and $B$ are distinct and positive, the only stable stationary point of the iteration defined by (\ref{eq:ourupdate}) is the eigenvector $w_i$ (up to sign).
\end{restatable}

\subsection{Defining Utilities and Pseudo-Utilities with Lagrangian Functions}

Now observe that the updates (\ref{eq:psdupdate}) can be written as the gradients of a Lagrangian pseudo-utility function:
\begin{align}\label{eq:lagrangeutils}
\mathcal{PU}_{i}^{\text{GHA-GEP-PSD}}(w_i | w_{j<i}, \bGam ) 
&=\tfrac{1}{2} \hat{w}_{i}^{\top}A\hat{w}_{i}
+\tfrac{1}{2} \textcolor{blue}{\Gamma_{ii}} (1 - \hat{w}_{i}^{\top}B\hat{w}_{i})
-\sum_{j< i} \textcolor{blue}{\Gamma_{ij}} \hat{w}_{j}^{\top}B\hat{w}_{i}.
\end{align}

We show how this result is closely related to the pseudo-utility functions in \citet{chen2019constrained} and suggests an alternative pseudo-utility function for the work in \citet{gemp2021} in Appendix \ref{sec:pseudo-utils} which, unlike the original work, does not require stop gradient operators.

If we plug in the relevant $w_i$ and $w_j$ terms into $\bGam$, we obtain the following utility function:
\begin{align}\label{eq:game-utility-psd}
\mathcal{U}_{i}^{\delta\text{-PSD}}(w_i ; w_{j<i}) 
&=\tfrac{1}{2}\hat{w}_{i}^{\top}A\hat{w}_{i}
+\tfrac{1}{2}\hat{w}_{i}^{\top}A\hat{w}_{i}\left(1-\hat{w}_{i}^{\top}B\hat{w}_{i}\right)
-\sum_{j< i} \hat{w}_{i}^{\top}A\hat{w}_{j} \hat{w}_{j}^{\top}B\hat{w}_{i} \nonumber \\
&= (\hat{w}_{i}^{\top}A\hat{w}_{i})
- \tfrac{1}{2} (\hat{w}_{i}^{\top}A\hat{w}_{i})(\hat{w}_{i}^{\top}B\hat{w}_{i})
- \sum_{j< i} (\hat{w}_{i}^{\top}A\hat{w}_{j})( \hat{w}_{j}^{\top}B\hat{w}_{i})
\end{align}

Again, we apply a modification to prevent blow-up when $A$ has negative eigenvalues, giving utility
\begin{align}\label{eq:game-utility}
\mathcal{U}_{i}^{\delta}(w_i ; w_{j<i})
&= (\hat{w}_{i}^{\top}A\hat{w}_{i})
- \tfrac{1}{2} \textcolor{violet}{\max(}(\hat{w}_{i}^{\top}A\hat{w}_{i})\textcolor{violet}{,0)}\,(\hat{w}_{i}^{\top}B\hat{w}_{i})
- \sum_{j< i} (\hat{w}_{i}^{\top}A\hat{w}_{j})( \hat{w}_{j}^{\top}B\hat{w}_{i})
\end{align}

A remarkable fact is that this utility function actually defines a solution to the GEP problem! We prove the following consistency result in Appendix \ref{nashproof}.

\begin{restatable}[Unique utility maximiser]{proprep}{deltanash}
\label{prop:delta-nash}
Assuming the top-$i$ generalized eigenvalues of the GEP (\ref{eq:consgep}) are positive and distinct. Then the unique maximizer of the utility in (\ref{eq:game-utility}) for exact parents is precisely the $i^{th}$ eigenvector (up to sign).
\end{restatable}

This utility function allows us to formalise $\Delta$-EigenGame, whose solution corresponds to the top-k solution of equation (\ref{eq:consgep}).

\begin{definition}
Let $\Delta$-EigenGame be the game with players $i \in \{1,...,k\}$, strategy space $\hat{w}_{i} \in \mathbb{R}^d$, where d is the dimensionality of A and B, and utilities $\mathcal{U}_{i}^{\delta}$ defined in equation (\ref{eq:game-utility}).
\end{definition}

An immediate corollary of Proposition \ref{prop:delta-nash} is:
\begin{corollary}
    The top-$k$ generalized eigenvectors form the unique, strict Nash equilibrium of $\Delta$-EigenGame.
\end{corollary}

Furthermore, the penalty terms in the utility function (\ref{eq:lagrangeutils}) have a natural interpretation as a projection deflation as shown in appendix \ref{sec:defl}.

Next note that it is easy to compute the derivative
\begin{align}\label{eq:utility-update}
    \Delta_{i}^\delta 
    &=
    \frac{\partial \mathcal{U}_{i}^{\delta}(w_i ; w_{j<i})}{\partial w_i} \\
    &=
    2 A\hat{w}_{i}
    - \{ A\hat{w}_{i}(\hat{w}_{i}^{\top}B\hat{w}_{i}) + (\hat{w}_{i}^{\top}A\hat{w}_{i})B \hat{w}_{i}\}
    - \sum_{j< i} \{A\hat{w}_{j}(\hat{w}_{j}^{\top}B\hat{w}_{i}) + (\hat{w}_{j}^{\top}A\hat{w}_{i})B \hat{w}_{j}\}\nonumber\\ 
    &=
    \Delta_i^\text{GHA-GEP} + \{ A\hat{w_i} - \sum_{j\leq i} A\hat{w}_{j}(\hat{w}_{j}^{\top}B\hat{w}_{i})\} \nonumber
\end{align}
We can use these gradients as updates step for an alternative algorithm for the GEP which we call $\delta$-EigenGame (where, consistent with previous work, we use upper case for the game and lower case for its associated algorithm). We can now discuss stochastic versions of the algorithms introduced above, the setting where our methods excel.

\subsection{Stochastic/Data-streaming versions}

This paper is motivated by cases where the algorithm only has access to unbiased sample estimates of $A$ and $B$. These estimates, denoted $\hat{A}$ and $\hat{B}$, are therefore random variables. A nice property of both our proposed GHA-GEP and $\delta$-EigenGame is that $A$ and $B$ appear as multiplications in both of their updates (as opposed to as divisors). This means that we can simply substitute them for our unbiased estimates at each iteration. For the GHA-GEP algorithm this gives us updates based on stochastic unbiased estimates of the gradient

\begin{align}
\hat{\Delta}_{i}^{\text{GHA-GEP}}=&\hat{A} \hat{w}_{i} - \sum_{j \leq i}\hat{B}\hat{w}_{j}\textcolor{blue}{\left(\hat{w}_{j}^{\top} \hat{A} \hat{w}_{i}\right)}.
\label{eq:ourstochasticupdate}
\end{align}

Which we can use to form algorithm \ref{alg:stochastic-GHA-GEP}.

\begin{algorithm}
   \caption{A Sample Based Generalized Hebbian Algorithm for GEP}
   \label{alg:stochastic-GHA-GEP}
\begin{algorithmic}
   \STATE {\bfseries Input:} data stream $Z_t$ consisting of $b$ samples from $z_n$. Learning rate $(\eta_t)_t$. Number of time steps $T$. Number of eigenvectors to compute $k$.
   \STATE {\bfseries Initialise:} $(\hat{w}_i)_{i=1}^K$ with random uniform entries
   \FOR{$t=1$ {\bfseries to} $T$}
    \STATE Construct independent unbiased estimates $\hat{A}$ and $\hat{B}$ from $Z_t$
    \FOR{$i=1$ {\bfseries to} $k$}
        \STATE $\hat{w}_{i} \leftarrow \hat{w}_{i}+\eta_{t} \hat{\Delta}_{i}^{\text{GHA-GEP}}$ 
        \COMMENT{As defined in (\ref{eq:ourstochasticupdate})}
    \ENDFOR
   \ENDFOR
\end{algorithmic}
\end{algorithm}

Likewise we can form stochastic updates for $\delta$-EigenGame

\begin{align}\label{eq:ourstochasticdeltaupdate}
    \Delta_{i}^\delta 
        &=
        2 \hat{A}\hat{w}_{i}
        - \{ \hat{A}\hat{w}_{i}(\hat{w}_{i}^{\top}\hat{B}\hat{w}_{i}) + (\hat{w}_{i}^{\top}\hat{A}\hat{w}_{i})\hat{B} \hat{w}_{i}\}
        - \sum_{j< i} \{\hat{A}\hat{w}_{j}(\hat{w}_{j}^{\top}\hat{B}\hat{w}_{i}) + (\hat{w}_{j}^{\top}\hat{A}\hat{w}_{i})\hat{B} \hat{w}_{j}\}
\end{align}

Which give us algorithm \ref{alg:stochastic-utility-algo}.

\begin{algorithm}
   \caption{A Sample Based $\delta$-EigenGame for GEP}
   \label{alg:stochastic-utility-algo}
\begin{algorithmic}
   \STATE {\bfseries Input:} data stream $Z_t$ consisting of $b$ samples from $z_n$. Learning rate $(\eta_t)_t$. Number of time steps $T$. Number of eigenvectors to compute $k$.
   \STATE {\bfseries Initialise:} $(\hat{w}_i)_{i=1}^K$ with random uniform entries
   \FOR{$t=1$ {\bfseries to} $T$}
       \STATE Construct independent unbiased estimates $\hat{A}$ and $\hat{B}$ from $Z_t$
        \FOR{$i=1$ {\bfseries to} $k$}
            \STATE $\hat{w}_{i} \leftarrow \hat{w}_{i}+\eta_{t} \Delta_{i}^{\delta}$ 
            \COMMENT{As defined in (\ref{eq:ourstochasticdeltaupdate})}
        \ENDFOR
   \ENDFOR
\end{algorithmic}
\end{algorithm}

Furthermore, the simplicity of the form of the updates means that, in contrast to previous work, our updates in the stochastic setting require only one hyperparameter - the learning rate.

\subsection{Complexity and Implementation}

For the GEPs we are motivated by, and in particular for CCA, $\hat{A}$ and $\hat{B}$ are low rank matrices (specifically, they have at most rank $b$ where $b$ is the mini-batch size). This means that, like previous variants of EigenGame, our algorithm has a per-iteration cost of $\mathcal{O}(bdk^2)$. We can similarly leverage parallel computing in both the eigenvectors (players) and data to achieve a theoretical complexity of $\mathcal{O}(dk)$. 

A particular benefit of our proposed form is that we only require one hyperparameter which makes hyperparameter tuning particularly efficient. This is particularly important as prior work has demonstrated that methods related to the stochastic power method are highly sensitive to the choice of learning rate \cite{li2021nonconvex}. Indeed, by using a decaying learning rate the user can in principle run our algorithm just once to a desired accuracy given their computational budget. This is in contrast to recent work proposing an EigenGame solution to stochastic GEPs \citep{gemp2022generalized} which requires three hyperparameters.

\section{Application to CCA and Extension to For Deep CCA}

Previous EigenGame approaches have not been extended to include deep learning functions. \citet{gemp20} noted that the objectives of the players in $\alpha$-EigenGame were all generalized inner products which should extend to general function approximators. However, it was unclear how to translate the constraints in previous EigenGame approaches to the neural network setting. In contrast, we have shown that our work is constraint free but can still be written completely as generalized inner products for certain GEPs and, in particular, dimensionality reduction methods like CCA. 

\subsection{Canonical Correlation Analysis}\label{sec:cca}
Suppose we have vector-valued random variables $X,Y \in \R^p,\R^q$ respectively. Then CCA \citep{hotelling1992relations} defines a sequence of pairs of `canonical directions' $(u_i,v_i)\in \R^{p+q}$ by the iterative maximisations
\begin{align}\label{eq:conscca}
\max_{u\in\R^p,v\in\R^q} \operatorname{Cov}(u^\top X,v^\top Y) 
\quad\text{ subject to } 
& \operatorname{Cov}(u^\top X) =  \operatorname{Cov}(v^\top Y) =1, \\
& \operatorname{Cov}(u^\top X,u_j^\top X) = \operatorname{Cov}(v^\top Y, v_j^\top Y)=0 \text{ for } j<i. \nonumber
\end{align}

Now write $\operatorname{Cov}(X)=\Sigma_{XX},\operatorname{Cov}(Y) = \Sigma_{YY}, \operatorname{Cov}(X,Y) = \Sigma_{XY}$. It is straightforward to show (\cite{borga_learning_1998}) that CCA corresponds to a GEP with 

\begin{equation}\label{eq:cca-GEV}
	A = \begin{pmatrix}0 &\Sigma_{XY} \\ \Sigma_{YX} & 0\end{pmatrix}, \qquad
	B = \begin{pmatrix}\Sigma_{XX} & 0 \\ 0 & \Sigma_{YY}\end{pmatrix}, \qquad
	w =\begin{pmatrix}	u \\ v\end{pmatrix}, \qquad
	d=p+q.
\end{equation}

For the sample version of CCA, suppose we have observations $(x_n,y_n)_{n=1}^N$, which have been pre-processed to have mean zero. Then the classical CCA estimator solves the GEV above with covariances replaced by sample covariances \cite{anderson_introduction_2003}. To define our algorithm in the stochastic case, suppose that at time step $t$ we define $\hat{A}_{t},\hat{B}_{t}$ by plugging sample covariances of the mini-batch at time $t$. 

\subsection{\texorpdfstring{$\delta$}{delta}-EigenGame for CCA}
We defined CCA by maximising correlation between linear functionals of the two views of data; we can extend this to DCCA by instead considering non-linear functionals defined by deep neural networks. Consider neural networks $f,g$ which respectively map $X$ and $Y$ to a $d$ dimensional subspace. We will refer to the $k^{\text{th}}$ dimension of these subspaces using $f_k(X)$ and $g_k(X)$ where $f(X)=[f_1(X), ..., f_d(X)]$ and $g(X)=[g_1(X), ..., g_d(X)]$. Deep CCA finds $f$ and $g$ which maximize $\operatorname{Corr}(f_i(X),g_i(Y))$ subject to orthogonality constraints. 

To motivate an algorithm, note that (\ref{eq:game-utility}) is just a function of the inner products
\begin{align*}
    \langle\hat{w}_{i},{A}\hat{w}_{j}\rangle
    &=\operatorname{Cov}(u_i^\top X,v_j^\top Y) + \operatorname{Cov}(v_i^\top Y, u_j^\top X)\\
    \langle\hat{w}_{i},{B}\hat{w}_{j}\rangle
    &=\operatorname{Cov}(u_i^\top X,u_j^\top X) + \operatorname{Cov}(v_i^\top Y, v_j^\top Y)
\end{align*}

So replacing $u_i^\top X$ with $f_i(X)$ and $v_i^\top Y$ with $g_i(Y)$, and using the short-hand

\begin{align}
    \tilde{A}_{ij}
    &=\operatorname{Cov}(f_i(X),g_j(Y)) + \operatorname{Cov}(g_i(Y), f_j(X))\\
    \tilde{B}_{ij}
    &=\operatorname{Cov}(f_i(X),f_j(X)) + \operatorname{Cov}(g_i(Y), g_j(Y))
\end{align}
we obtain the objective
\begin{align}
\mathcal{U}_{i}^{\delta}(f_i,g_i |f_{j<i},g_{j<i}) &=2 \tilde{A}_{ii} - \tilde{A}_{ii}\tilde{B}_{ii} - 2 \sum_{j<i} \tilde{A}_{ij}\tilde{B}_{ij}
\end{align}

Next observe by symmetry of matrices $\tilde{A},\tilde{B}$ that if we sum the first $k$ utilities we obtain
\begin{align}
U_k^\text{sum} 
= \sum_{i=1}^k U_i^\delta
&= \sum_{i=1}^k 2 \tilde{A}_{ii} - \sum_{i=1}^k \tilde{A}_{ii} \tilde{B}_{ii} - 2 \sum_{i=1}^k \sum_{j<i} \tilde{A}_{ij}\tilde{B}_{ij} \nonumber \\
&= 2 \tr(\tilde{A}) - \sum_{i,j=1}^k \tilde{A}_{ij}\tilde{B}_{ij} \nonumber\\
&= 2 \tr(\tilde{A}) - \tr(\tilde{A} \tilde{B}^{\top})
= \tr\left(\tilde{A}(2 I_k - \tilde{B})\right) \label{eq:dcca-usum}
\end{align}

The key strength of this covariance based formulation is that we can obtain a full-batch algorithm by simply plugging in the sample covariance over the full batch; and obtain a mini-batch update by plugging in sample covariances on the mini-batch. We define DCCA-EigenGame in algorithm \ref{alg:dccaeigengame}, where we slightly abuse notation: we write mini-batches in matrix form $X_t\in\R^{p \times b},Y_t\in\R^{q\times b}$ and use short hand $f(X_t),f(Y_t)$ to denote applying $f,g$ to each sample in the mini-batch.

\begin{algorithm}[H]
   \caption{DCCA EigenGame}
   \label{alg:dccaeigengame}
\begin{algorithmic}
   \STATE {\bfseries Input:} Stream of data with mini-batch size $b$ $\left(X_t\in \mathbb{R}^{b\times p},Y_t\in \mathbb{R}^{b\times q})\right)$, neural networks $f(X)$, $g(Y)$ parameterized by $\hat{\theta}$ and $\hat{\psi}$, learning rate $\eta$
   \FOR{$t=1$ {\bfseries to} $T$}
        \STATE Construct unbiased estimates $\tilde{A}$ and $\tilde{B}$ from $f(X_t)$ and $g(Y_t)$
        \STATE $\mathcal{U} \leftarrow \tr\left(\tilde{A}(2 I_k - \tilde{B})\right)$
        \STATE $\tilde{\nabla}_{f} \leftarrow \frac{\partial U}{\partial f}, \tilde{\nabla}_{g} \leftarrow \frac{\partial U}{\partial g}$
        \STATE $\hat{\theta}_{t+1} \leftarrow \hat{\theta}_{t}+\eta \tilde{\nabla}_{f}, \hat{\psi}_{t+1} \leftarrow \hat{\psi}_{t}+\eta \tilde{\nabla}_{g}$
   \ENDFOR
\end{algorithmic}
\end{algorithm}

We have motivated a loss function for SGD by a heuristic argument. We now give a theoretical result justifying the choice. Recall  the top-$k$ variational characterisation of the GEP in (\ref{eq:subspace-gep}) was hard to use in practice because of the constraints; we can use this to prove that the form above characterises the GEP. 

\begin{restatable}[Subspace characterisation]{proprep}{subspacecharac}
\label{prop:subspace-charac}
Let $A$ be positive semi-definite. Then the top-$k$ subspace for the GEP (\ref{eq:igep}) can be characterised by 
    \begin{equation}\label{eq:subspace-unconstrained-program}
    \max_{W \in \R^{d\times k}} \tr\left( W^{\top} A W \: (2 \, I_k - W^{\top} B W) \right)
    \end{equation}
\end{restatable}

We prove this result in Appendix \ref{sec:subspace-charac}. We also provide an alternative derivation of the utility of (\ref{eq:dcca-usum}) from the paper of \cite{wang2015stochastic} in Appendix \ref{sec:dcca-alt-deriv}.

\section{Related Work}

In particular we note the contemporaneous work in \citet{gemp2022generalized}, termed $\gamma$-EigenGame, which directly addresses the stochastic GEP setting we have described in this work using an EigenGame-inspired approach. Since their method was designed around the Rayleigh quotient form of GEPs, it takes a different and more complicated form and requires additional hyperparameters in order to remove bias from the updates in the stochastic setting due to their proposed utility function containing random variables in denominator terms. It also isn't clear that their updates are the gradients of a utility function. \citet{meng2021online} developed an algorithm, termed RSG+, for streaming CCA which stochastically approximates the principal components of each view in order to approximate the top-k CCA problem, in effect transforming the data so that $B=I$ to simplify the problem. \citet{arora2017stochastic} developed a Matrix Stochastic Gradient method for finding the top-k CCA subspace. However, the efficiency of this method depends on mini-batch samples of 1 and scales poorly to larger mini-batch sizes. While there have also been a number of approaches to the top-1 CCA problem \citep{li2021nonconvex,bhatia2018gen},
the closest methods in motivation and performance to our work on the linear problem are $\gamma$-EigenGame, SGHA, and RSG+.

The original DCCA \citep{andrew2013deep} was defined by the objective

\begin{align}\label{eq:dccaobj}
\max\operatorname{tracenorm}(\hat{\Sigma}_{XX}^{-1 / 2} \hat{\Sigma}_{XY} \hat{\Sigma}_{YY}^{-1 / 2})
\end{align}

and demonstrated strong performance in multiview learning tasks when optimized with the full batch L-BFGS optimizer \citep{liu1989limited}. However when the objective is evaluated for small mini-batches, the whitening matrices $\hat{\Sigma}_{XX}^{-1 / 2}$ and $\hat{\Sigma}_{YY}^{-1 / 2}$ are likely to be ill-conditioned, causing gradient estimation to be biased.

\citet{wang2015unsupervised} observed that despite the biased gradients, the original DCCA objective could still be used in the stochastic setting for large enough mini-batches, a method referred to in the literature as stochastic optimization with large mini-batches (DCCA-STOL). \citet{wang2015stochastic} developed a method which adaptively approximated the covariance of the embedding for each view in order to whiten the targets of a regression in each view. This mean square error type loss can then be decoupled across samples in a method called non-linear orthogonal iterations (DCCA-NOI). To the best of our knowledge this method is the current state-of-the-art for DCCA optimisation using stochastic mini-batches. 

\section{Experiments}
\label{Experiments}

In this section we replicate experiments from recent work on stochastic CCA and Deep CCA in order to demonstrate the accuracy and efficiency of our method.

\subsection{Stochastic Solutions to CCA}

In this section we compare GHA-GEP and $\delta$-EigenGame to previous methods for approximating CCA in the stochastic setting. We optimize for the top-8 eigenvectors for the MediaMill, Split MNIST and Split CIFAR datasets, replicating \citet{gemp2022generalized,meng2021online} with double the number of components and mini-batch size 128 and comparing our method to theirs. We use the Scipy \citep{virtanen2020scipy} package to solve the population GEPs as a ground truth value and use the proportion of correlation captured (PCC) captured by the learnt subspace as compared to this population ground truth (defined in Appendix \ref{sec:pccdef}). 

Figure \ref{fig:ccaiter} shows that for all three datasets, both GHA-GEP and $\delta$-EigenGame exhibit faster convergence on both a per-iteration basis compared to prior work and likewise in terms of runtime in figure \ref{fig:ccatime}. They also demonstrate comparable or higher PCC at convergence. In these experiments $\delta$-EigenGame was found to outperform GHA-GEP. These results were broadly consistent across mini-batch sizes from 32 to 128 which we demonstrate in further experiments in Appendix \ref{sec:ccaextra}.

The strong performance of GHA-GEP and $\delta$-EigenGame is likely to be because their updates adaptively weight the objective and constraints of the problem and are not constrained arbitrarily to the unit sphere. We further explore the shape of the utility function in Appendix \ref{sec:utilityshape}.

\begin{figure}
     \centering
     \begin{subfigure}[b]{0.32\textwidth}
         \centering
         \includegraphics[width=\textwidth]{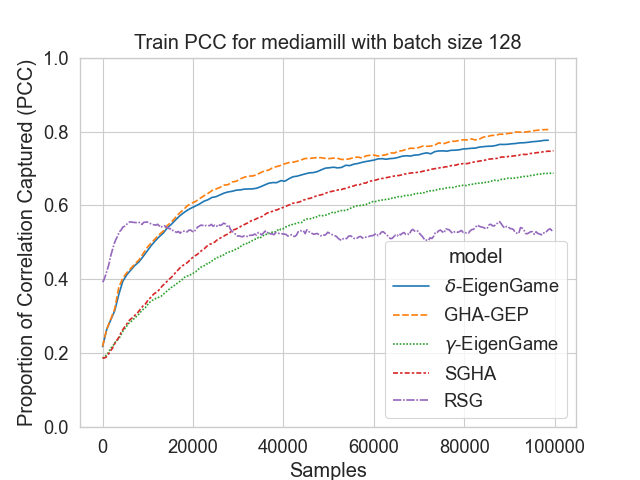}
         \label{fig:ccamediamill}
     \end{subfigure}
     \hfill
     \begin{subfigure}[b]{0.32\textwidth}
         \centering
         \includegraphics[width=\textwidth]{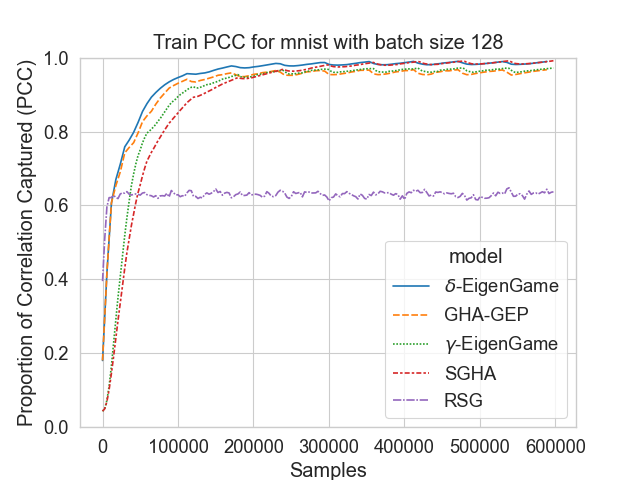}
         \label{fig:ccamnist}
     \end{subfigure}
     \hfill
     \begin{subfigure}[b]{0.32\textwidth}
         \centering
         \includegraphics[width=\textwidth]{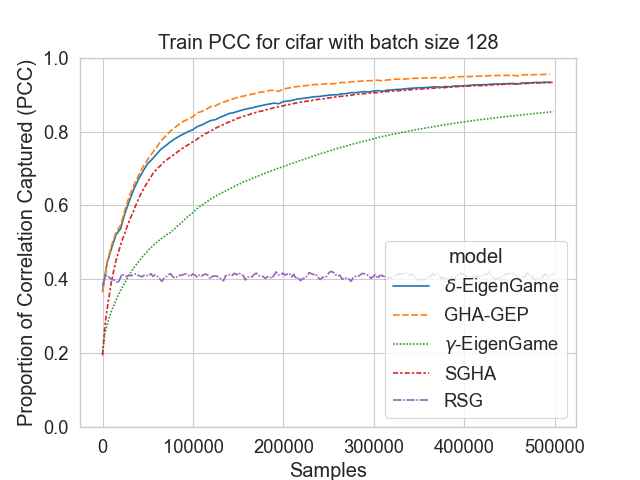}
         \label{fig:ccacifar}
     \end{subfigure}
        \caption{CCA with stochastic mini-batches: proportion of correlation captured with respect to Scipy ground truth b yGHA-GEP and $\delta$-EigenGame vs prior work. The maximum value is 1.}
        \label{fig:ccaiter}
     \centering
     \begin{subfigure}[b]{0.32\textwidth}
         \centering
         \includegraphics[width=\textwidth]{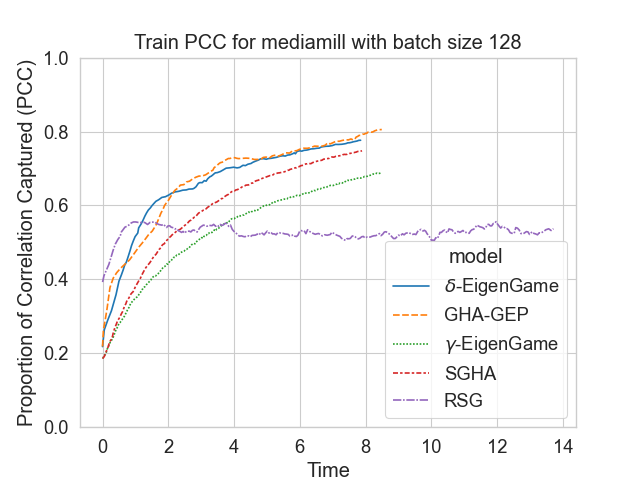}
     \end{subfigure}
     \hfill
     \begin{subfigure}[b]{0.32\textwidth}
         \centering
         \includegraphics[width=\textwidth]{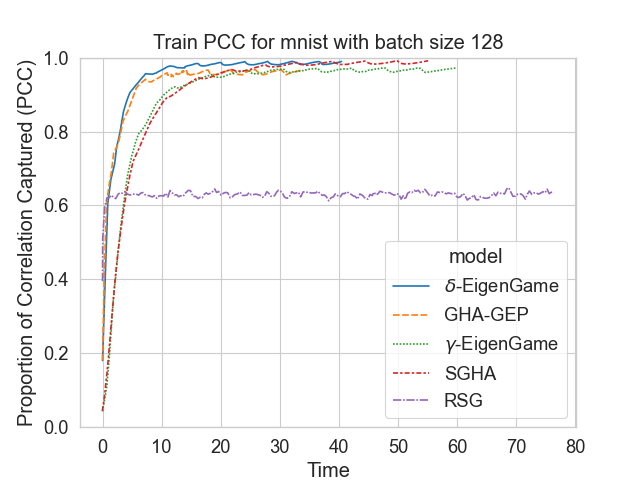}
     \end{subfigure}
     \hfill
     \begin{subfigure}[b]{0.32\textwidth}
         \centering
         \includegraphics[width=\textwidth]{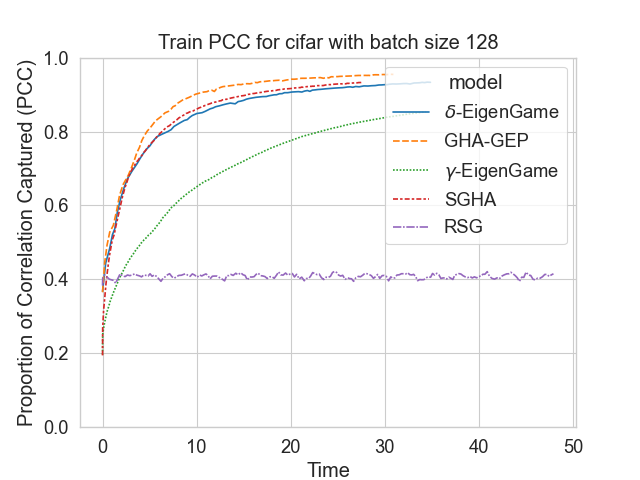}
     \end{subfigure}
        \caption{CCA with stochastic mini-batches: proportion of correlation captured with respect to Scipy ground truth by GHA-GEP and $\delta$-EigenGame vs prior work. The maximum value is 1.}
        \label{fig:ccatime}
\end{figure}

\subsection{Stochastic Solutions to Deep CCA}

In this section we compare DCCA-EigenGame and DCCA-SGHA to previous methods for optimizing DCCA in the stochastic setting. We replicated an experiment from \citet{wang2015stochastic} and compare our proposed methods to DCCA-NOI and DCCA-STOL. Like previous work, we use the total correlation captured (TCC) of the learnt subspace as a metric (defined in Appendix \ref{sec:tccdef}).

In all three datasets, figure \ref{fig:dcca:split} shows that DCCA-EigenGame finds higher correlations in the validation data than all methods except DCCA-STOL with $n=500$ with typically faster convergence in early iterations compared to DCCA-NOI.

\begin{figure}
     \centering
     \begin{subfigure}[b]{0.31\textwidth}
         \centering
         \includegraphics[width=\textwidth]{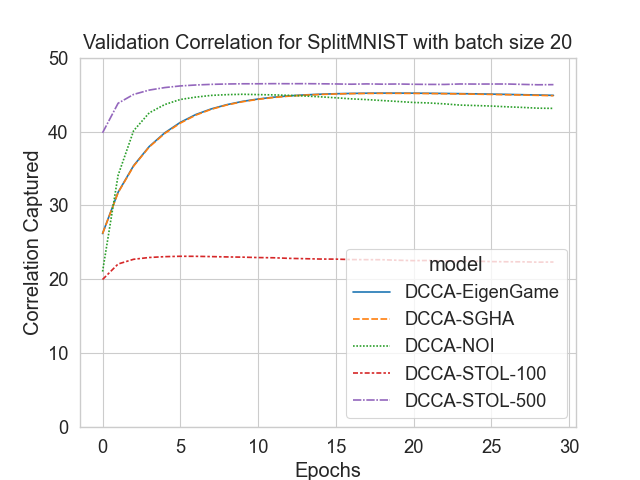}
     \end{subfigure}
     \hfill
     \begin{subfigure}[b]{0.31\textwidth}
         \centering
         \includegraphics[width=\textwidth]{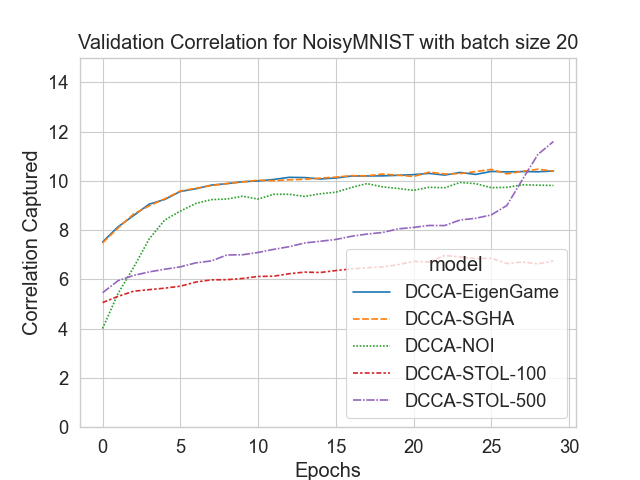}
     \end{subfigure}
     \hfill
     \begin{subfigure}[b]{0.31\textwidth}
         \centering
         \includegraphics[width=\textwidth]{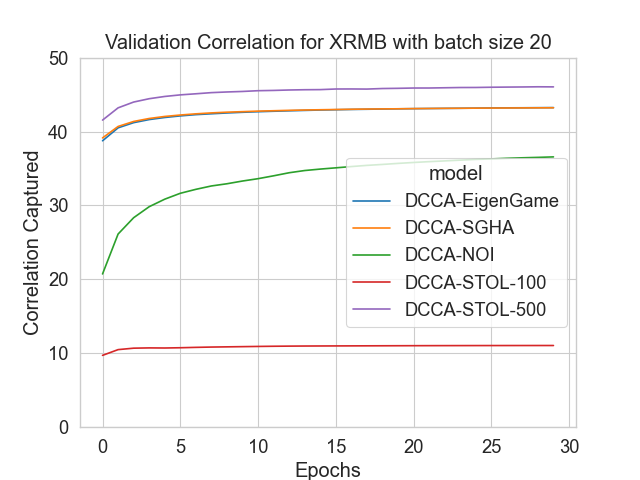}
     \end{subfigure}
     \newline
     \centering
     \begin{subfigure}[b]{0.31\textwidth}
         \centering
         \includegraphics[width=\textwidth]{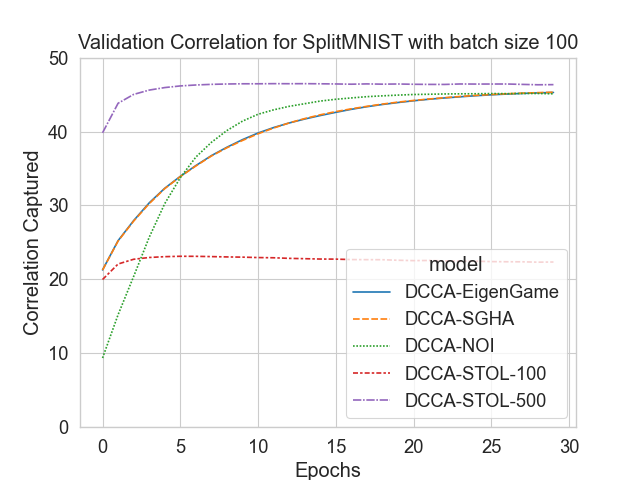}
     \end{subfigure}
     \hfill
     \begin{subfigure}[b]{0.31\textwidth}
         \centering
         \includegraphics[width=\textwidth]{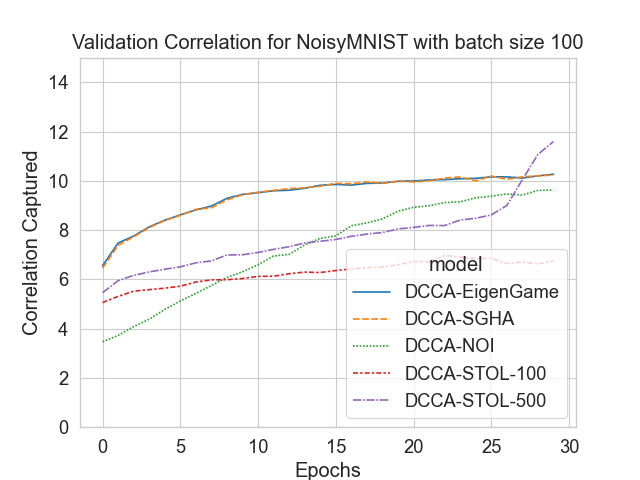}
     \end{subfigure}
     \hfill
     \begin{subfigure}[b]{0.31\textwidth}
         \centering
         \includegraphics[width=\textwidth]{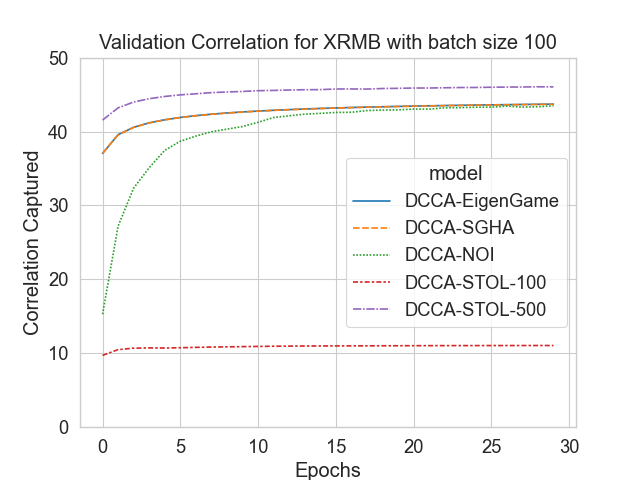}
     \end{subfigure}
        \caption{Total correlation captured by the 50 latent dimensions in the validation data. The maximum value is 50. The top row show results for  mini-batch size 20 and the bottom row show results for mini-batch size 100}
        \label{fig:dcca:split}
\end{figure}

\section{Conclusion}
\label{Conclusion}

We have presented two novel algorithms for optimizing stochastic GEPs. The first, GHA-GEP was based on extending the popular GHA and we showed how it could be understood as optimising a Lagrangian psuedo-utility function. The second, $\delta$-EigenGame, was developed by swapping the Lagrange multipliers to give a proper utility function which allowed us to define the solution of a GEP with $\Delta$-EigenGame. Our proposed methods have simple and elegant forms and require only one choice of hyperparameter, making them extremely practical and both demonstrated comparable or better runtime and performance as compared to prior work.

We also showed how this approach can also be used to optimize Deep CCA and demonstrated state-of-the-art performance when using stochastic mini-batches. We believe that this will allow researchers to apply DCCA to a much wider range of problems.

In future work, we will apply $\delta$-EigenGame to other practically interesting GEPs like Generalized CCA for more than two views and Fisher Discriminant Analysis. We will also explore the extensions of other GEPs to the deep learning case in order to build principled deep representations.

\bibliography{main}
\bibliographystyle{iclr2023_conference}

\appendix

\section{Comparison to Previous Work}\label{sec:previousworkcomparison}
\subsection{Generalized Hebbian Algorithm}\label{sec:gha}

Our update is closely related to the Generalized Hebbian Algorithm (GHA) \citep{sanger1989optimal} for solving the PCA problem with updates:

\begin{align}
\Delta_{i}^{\text{GHA}}=A \hat{w}_{i} - \sum_{j \leq i}\hat{w}_{j}\textcolor{blue}{\left(\hat{w}_{j}^{\top} A \hat{w}_{i}\right)} = A \hat{w}_{i} - \sum_{j \leq i}\hat{w}_{j}\textcolor{blue}{\Gamma_{ij}} \label{eq:sanger}
\end{align}

which was originally designed to be solved sequentially rather than in parallel. Note that for GEPs where $B=I_d$ like PCA, our proposed method collapses exactly to GHA. 

\subsection{Stochastic Generalized Hebbian Algorithm}\label{sec:sgha}

To understand how our method extends GHA to generalized eigenvalue problems we consider the Stochastic Generalized Hebbian Algorithm (SGHA) \citep{chen2019constrained}\footnote{Though the authors called this SGHA, it is rather different to the original proposal, because it is a subspace method rather than an iterative one}. SGHA is derived from the min-max Lagrangian form of (\ref{eq:subspace-gep}):

\begin{align}\label{eq:minmax}
\min_{W \in \R^{d\times k}} \max_{\Gamma\in \R^{k\times k}} \mathcal{L}(W, \textcolor{blue}{\Gamma})=-\operatorname{tr}\left(W^{\top} A W\right)+\left\langle \textcolor{blue}{\Gamma}, W^{\top} B W-I_k\right\rangle
\end{align}

Where $W$ is a matrix that captures the top-$k$ subspace (but not necessarily the top-k eigenvectors) and $\textcolor{blue}{\Gamma}$ is a Lagrange multiplier that enforces the constraint in (\ref{eq:consgep}) along with the $B$-orthogonality of each eigenvector. By solving for the KKT conditions of equation (\ref{eq:minmax}), we have $\textcolor{blue}{\Gamma=\left(W^{\top} A W\right)}$ and the authors propose to combine the primal and dual updates into a single step to give symmetrical updates for each eigenvector:

\begin{align}
\Delta_{i}^{\text{SGHA}}&=A \hat{w}_{i}-\sum_{\textcolor{red}{j}}B\hat{w}_{j}\textcolor{blue}{\left(\hat{w}_{j}^{\top} A \hat{w}_{i}\right)} = A \hat{w}_{i}-\sum_{\textcolor{red}{j}}B\hat{w}_{j}\textcolor{blue}{\Gamma_{ij}}
\end{align}

Where we highlight in red the key difference between our method and SGHA: that there is no hierarchy imposed on eigenvectors, so method can only recover top-$k$ subspace; this is in contrast to our proposal, all eigengame methods, and indeed the original GHA method. As noted by \citet{gemp20}, imposing a hierarchy often appears to improve the stability of the algorithm in experiments and has the additional benefit of returning ordered eigenvectors.

\subsection{\texorpdfstring{$\mu$-EigenGame}{mu-EigenGame}}\label{sec:mueigengame}

Finally, our method is closely related to $\mu$-EigenGame \cite{gemp2021} - though this is only defined for the $B=I_d$ case. Their method restricts estimates to lie on the unit sphere, using Riemannian optimization tools to update in directions defined by 
\begin{align}
\Delta_{i}^{\mu}=A \hat{w}_{i}-\sum_{\textcolor{red}{j<i}} \hat{w}_{j}\textcolor{blue}{\left(\hat{w}_{j}^{\top} A \hat{w}_{i}\right)} = A \hat{w}_{i}-\sum_{\textcolor{red}{j<i}} \hat{w}_{j}\textcolor{blue}{\Gamma_{ij}}
\end{align}
where we again highlight in red the difference compared to our proposal: $\mu$-EigenGame however does not have the $j=i$ term in its penalty (and therefore does not use the \textcolor{blue}{$\Gamma_{ii}$} Lagrange multipliers associated with the unit variance constraint $\hat{w}_{i}^{\top}\hat{w}_{i}=1$).

\section{Proofs and further analysis for sequential algorithms}
\subsection{New Notation}\label{sec:proof-ntn}
For the rest of this section we will find it convenient to introduce some new notation. Firstly, we change notation from the main text and index the normalised solutions to the GEV with superscripts:
$$\langle w^{(i)}, Bw^{(i)}\rangle=1, \quad \langle w^{(i)}, Aw^{(i)}\rangle=\lambda^{(i)} \: \forall{i},$$
while we continue to index our estimates with subscripts. We can write our estimates in this basis to define the coefficients $\hat{w}_{i}=\sum_p \nu^{(p)}_iw^{(p)}$. 
First note this gives
\begin{align*}
    \langle\hat{w}_{i}, Aw^{(j)}\rangle &= \lambda^{(j)} \nu_i^{(j)}\\ \vspace{14pt}
    \langle\hat{w}_{i}, Bw^{(j)}\rangle &= \nu_i^{(j)} 
\end{align*}
It is also convenient to write the matrix $\Lambda = \text{diag}(\{\lambda^{(p)}\}_p)$ and the vector $\nu_i = (\nu_i^{(p)})_{p=1}^d$. We then get the following nice identities:
\begin{align*}
\langle\hat{w}_{i}, A\hat{w}_{j}\rangle &= \sum_p \lambda^{(p)} \nu^{(p)}_i \nu^{(p)}_j = \nu_i^T \Lambda \nu_j\\
\langle\hat{w}_{i}, B\hat{w}_{j}\rangle &= \sum_p \nu^{(p)}_i \nu^{(p)}_j = \nu_i^T \nu_j
\end{align*}

Next define
\begin{align*}
    m_i &= \|\nu_i\|_2^2 = \sum_p (\nu^{(p)}_i)^2 \\
    z^{(j)}_i &=\frac{(\nu^{(j)}_i)^2}{m_i} = \frac{(\nu^{(j)}_i)^2}{\sum_p (\nu^{(p)}_i)^2} \\[6pt]
    \bar{\lambda}_i &= \sum_{j} z_i^{(j)} \lambda^{(j)} = \frac{\nu_i^T \Lambda \nu_i}{m_i} 
\end{align*}
so that the vector $z_i = (z^{(j)}_i)_j$ takes values in the simplex. We can think of $m_i$ as a magnitude, $z_i$ as a direction, and $\bar{\lambda}_i$ as an appropriate `average' value of $\lambda$.
We will see that the utility functions and dynamics analysed below often decompose nicely as functions of these quantities.

\subsection{\texorpdfstring{$\Delta$}{Delta}-eigengame Theory}\label{nashproof} 
We recall \cref{prop:delta-nash}:
\deltanash*

\begin{proof}[Proof of special case, where $A$ is positive semi-definite]
    Since $A$ is positive semi-definite, (\ref{eq:game-utility-psd}) reduces to (\ref{eq:game-utility}). 
    
    First note that using the notation in Section \ref{sec:proof-ntn} 
    \begin{align*}    
        \langle\hat{w}_{i}, A\hat{w}_{i}\rangle &= m_i \bar{\lambda_i} = m_i \sum_j z_i^{(j)}\lambda^{(j)} \\
        \langle\hat{w}_{i}, B \hat{w}_{i}\rangle &= \nu_i^T \nu_i = m_i        
    \end{align*}
    while with exact parents  $\hat{w}_j = w^{(j)}$ for $j<i$ we get
    \begin{align*}
        \langle\hat{w}_{i},B\hat{w}_{j}\rangle\langle\hat{w}_{j},A\hat{w}_{i}\rangle = \lambda^{(j)}(\nu_i^{j})^2 = m_i \lambda^{(j)}  z_i^{(j)}
    \end{align*}
    Thus we can simplify the utility function for player $i$ to:
    \begin{align}
    u_{i}(\hat{w}_{i}|\hat{w}_{j<i})
    &=
    2\langle\hat{w}_{i},A\hat{w}_{i}\rangle-2\sum_{j< i}\langle\hat{w}_{i},B\hat{w}_{j}\rangle\langle\hat{w}_{j},A\hat{w}_{i}\rangle -\langle\hat{w}_{i},B\hat{w}_{i}\rangle\langle\hat{w}_{i},A\hat{w}_{i}\rangle \nonumber \\
    &=2\sum_j m_i \lambda^{(j)} z_i^{(j)} - 2\sum_{j<i}m_i \lambda^{(j)} z_i^{(j)} - m_i \left(\sum_j m_i \lambda^{(j)} z_i^{(j)}\right) \label{eq:utility-simplex-form}\\
    &=2\sum_{j\geq i}m_i \lambda^{(j)} z_i^{(j)} - \sum_j m_i^2 \lambda^{(j)} z_i^{(j)}\nonumber\\
    &=(2m_i-m_{i}^2)\sum_{j\geq i}\lambda^{(j)} z^{(j)}_i - m_{i}^2\sum_{j<i}\lambda^{(j)}z^{(j)}_i \nonumber
    \end{align}
    
    We claim this is maximized precisely when $m_i=1$ $z_i^{(j)}=\delta_{ij}$; this corresponds to $\nu_i^{(i)}=\pm1$ and $\nu_i^{(j)}=0$ for $j\neq i$ and so indeed $\hat{w}_i = \pm w^{(i)}$. To prove the claim we have to consider two cases:
    \begin{itemize}
        \item \textbf{$m_i \in [0,2]$ case:} then $2m_i - m_i^2 > 0$. The objective is just a linear function of $z_i$, with maximal coefficient $(2m_i - m_i^2) \lambda^{(i)} > 0$, so is maximised by taking $z_i^{(j)}=\delta_{ij}$. The utility then reduces to $(2m_i - m_i^2) \lambda^{(i)}$ so is maximised by taking $m_i=1$, giving maximal value of $\lambda^{(i)}$ (assumed strictly positive). Note that this argument still holds even if some of the later $\lambda^{(j)}$'s are negative (recall the first $i$ eigenvalues were assumed positive). 
        \item \textbf{$m_i > 2$ case: } now $2m_i - m_i^2 < 0$ so the coefficients of $z_i^{(j)}$ are all negative. The utility is therefore also negative.
    \end{itemize}
    So the maximum is attained in the case $m_i \in [0,2]$, as required.
\end{proof}
\begin{proof}[Proof of general case, where $A$ may have negative eigenvalues]
    Instead of (\ref{eq:utility-simplex-form}) we now have
    \begin{align*}
        u_{i}(\hat{w}_{i}|\hat{w}_{j<i})
        &=2\sum_{j\geq i} m_i \lambda^{(j)} z_i^{(j)}  - m_i \: \ppart{\sum_j m_i \lambda^{(j)} z_i^{(j)} }
    \end{align*}    
    Note the term inside the maximum function is just $\bar{\lambda}_i m_i$, so has the same sign as $\bar{\lambda}_i$. When this is positive, we just recover the previous expression for utility. We can therefore write
    \begin{align*}\vspace{5pt}
    u_{i}(\hat{w}_{i}|\hat{w}_{j<i})
    =&\ind{\bar{\lambda}_i > 0} \bigl[(2m_i-m_{i}^2)\sum_{j\geq i}\lambda^{(j)} z^{(j)}_i - m_{i}^2\sum_{j<i}\lambda^{(j)}z^{(j)}_i \bigr]\\
    & + \ind{\bar{\lambda}_i \leq 0} [ 2 m_i \sum_{j\geq i} \lambda^{(j)} z_i^{(j)}]
    \end{align*}
    So to handle negative eigenvalues, we require this further split of our optimisation region (now by the values of $z_i$, which determine $\bar{\lambda}_i$, rather than of $m_i$).
    \begin{itemize}
        \item \textbf{$\bar{\lambda}_i>0$ case:} the objective reduces to the objective of the previous case. Our analysis of the case $m_i \in [0,2]$ goes through unchanged (by previous remark) and the previous maximiser $m_i=1,z_i^{(j)}=\delta_{ij}$ does satisfy $\bar{\lambda}_i = \lambda^{(i)}>0$.
        
        In the case where $m_i > 2$, more care is needed. The key observation is that $-m_i^2 < 2m_i - m_i^2$. Hence we again have
        \begin{align*}
            u_i(\hat{w}_i|\hat{w}_{j<i})
            &=
            (2m_i-m_{i}^2)\sum_{j\geq i}\lambda^{(j)} z^{(j)}_i - m_{i}^2\sum_{j<i}\lambda^{(j)}z^{(j)}_i \\
            &\leq 
            (2m_i - m_i^2) \sum_{j}\lambda^{(j)} z^{(j)}_i  \\
            &= (2m_i - m_i^2) \bar{\lambda}_i  
            \leq 0
        \end{align*}
        so we see the maximiser in this combined case is as claimed.
        \item \textbf{$\bar{\lambda}_i\leq 0$ case:} again we show the utility is negative in this regime. Again we use $\bar{\lambda}_i = \sum_{j} \lambda^{(j)} z_i^{(j)}$. Now, because we assumed the first $i$ eigenvalues were positive, we see
        \begin{equation}
            \sum_{j\geq i}  \lambda^{(j)} z_i^{(j)} = \bar{\lambda}_i - \underbrace{\sum_{j<i}  \lambda^{(j)} z_i^{(j)}}_{\geq 0} \leq \bar{\lambda}_i \leq 0
        \end{equation}
        so the utility $2m_i \sum_{j\geq i}  \lambda^{(j)} z_i^{(j)} \leq 0$ also.
    \end{itemize}
    Again, combining the cases we see the maximum value is $\lambda^{(i)}$, and is attained precisely when 
    $m_i=1$ $z_i^{(j)}=\delta_{ij}$.
\end{proof}

Note that in the previous lemma, the utility function took a simple form when we chose the true generalised eigenvectors as a basis; indeed, when using coefficients with respect to the basis, the utility only depended on the generalised eigenvalues and not the basis itself. This simple form shows how our utility interacts naturally with the geometry of the GEV problem. We will now analyse the corresponding simple form of the update steps.

\subsection{GHA-GEP Theory}\label{app:gha-gep}
We recall \cref{prop:gep-gha-stat}:
\gepghastationary*
\begin{proof}[Proof in the special case where $A$ is positive definite]
Again we use the notation of Section \ref{sec:proof-ntn}.
By positive definiteness the update (\ref{eq:ourupdate}) reduces to (\ref{eq:psdupdate}). So we get
\begin{align*}
    \Delta_{i}^{\delta} 
    &= 
    A \hat{w}_{i} - B\hat{w}_{i}\left(\hat{w}_{i}^{\top} A \hat{w}_{i}\right) - \sum_{j < i}B\hat{w}_{j}\left(\hat{w}_{j}^{\top} A \hat{w}_{i}\right) \\
    &=
    A \hat{w}_{i} - \sum_{j \leq i}B\hat{w}_{j}\left(\hat{w}_{j}^{\top} A \hat{w}_{i}\right) \\
    &=
    \sum_p \lambda^{(p)} \nu_i^{(p)} Bw^{(p)} - \sum_{j\leq i} \left(\sum_p \nu_j^{(p)}Bw^{(p)}\right)\left(\nu_i^T \Lambda \nu_j\right) \\[6pt]
    &= 
    \sum_p Bw^{(p)} \Bigl\{ \lambda^{(p)} \nu_i^{(p)} - \sum_{j\leq i} \nu_j^{(p)} \left(\nu_i^T \Lambda \nu_j\right) \Bigr\}
\end{align*}

Now by orthonormality of the $w^{(p)}$, we can equate coefficients and combine to vector form to obtain the update step for the coefficient vectors, which we shall notate by
\begin{align}
    \Delta(\nu)_i 
    &= \Lambda \nu_i - \sum_{j\leq i} \nu_j (\nu_j^{\top} \Lambda \nu_i) \nonumber \\ 
    &= \left(I - \sum_{j < i} \nu_j \nu_j^{\top} \right) \Lambda \nu_i - \nu_i (\nu_i^{\top} \Lambda \nu_i) \label{eq:gen-w-update}
\end{align}

Only now do we consider the assumption of exact parents. This corresponds to the coefficient vectors $\nu_j = e_j \forall j<i$ where $e_j$ is the $j$th unit vector ($(e_j)_k = \delta_{jk}$). Then
$$ \left(I - \sum_{j < i} \nu_j \nu_j^{\top} \right) 
= \begin{pmatrix}
0_{(i-1)\times(i-1)} & 0_{(i-1)\times(i-1)} \\
0_{(d-i+1) \times (i-1)} & I_{(d-i+1) \times (d-i+1)} \\
\end{pmatrix}
$$
So writing $\nu_i^{\top} \Lambda \nu_i = m_i \bar{\lambda}_i$ as before, the update equations for our coefficients become:
\begin{align}
    \Delta(\nu)_i^{(p)} &= - m_i \bar{\lambda}_i \nu_i^{(p)} \quad \text{for }p<i \label{eq:coeff-update1}\\ 
    \Delta(\nu)_i^{(p)} &= (\lambda^{(p)} - m_i \bar{\lambda}_i) \nu_i^{(p)} \quad \text{for }p\geq i \label{eq:coeff-update2}
\end{align}
So if we observe the qualitative behaviour:
\begin{itemize}
    \item For $p<i$ the coefficients are shrunk towards zero (for sufficiently small step sizes)
    \item For $p\geq i$ the coefficients grow/decay depending on their generalised eigenvalue. The larger the eigenvalue, the more the components grow / the less they shrink. So over time only the $i$th component will be selected.
    \item The overall magnitude of the solution shrinks faster when $\bar{\lambda}_i$ is large.
\end{itemize}
A stationary point of the iteration therefore requires $\nu_i^{(p)}=0$ for $p<i$. Then for each $p \geq i$ we must have either $m_i \bar{\lambda}_i=\lambda^{(p)}$ or $\nu_i^{(p)}=0$. Furthermore, $\nu_i^{(i)}$ grows at a faster rate than any of the other components, so provided this was non-zero at initialisation, it will be extracted uniquely. Finally, if $\nu_i^{(p)}=0 \:\forall p \neq i$ but $\nu_i^{(i)}\neq 0$ then by definition of $\bar{\lambda}_i$ we must have $\bar{\lambda}_i = \lambda^{(i)} $ but then zeroing (\ref{eq:coeff-update2}) requires $m_i \bar{\lambda}_i = \lambda^{(i)}$; so combining we get $m_i=1$ and so $\nu_i^{(i)}=\pm 1$, as required.
\end{proof}

\begin{proof}[Proof in the general case, where $A$ may have negative eigenvalues]
    The steps above go through verbatim if we replace the $\nu_i^T \Lambda \nu_i$ terms with $\ppart{\nu_i^T \Lambda \nu_i}$, and correspondingly replace $\bar{\lambda}_i$ terms with $\ppart{\bar{\lambda}_i}$.
\end{proof}

\subsubsection{Discussion of continuous dynamics}
In particular note that in the continuous time case above with exact parents, we can write the solutions to $\tfrac{d}{dt}\nu_i(t) = \Delta(\nu)_i$ with $\Delta(\nu)_i$ as in (\ref{eq:coeff-update1},\ref{eq:coeff-update2}) as
\begin{equation*}
    \nu_i^{(p)}(t) =  \nu_i^{(p)}(0) \exp{\left( \ind{\{p\geq i\}} \lambda^{(p)}t  -  \int_{s=0}^t  m_i(s)\, \ppart{\bar{\lambda}_i(s)}  ds\right)}
\end{equation*}
So when $z_i^{(i)}(0)\neq 0$ (hopefully this is almost sure), the trajectories on the simplex satisfy
\begin{equation*}
    \frac{z_i^{(j)}(t)}{z_i^{(i)}(t)} = \frac{z_i^{(j)}(0)}{z_i^{(i)}(0)} \exp{\left(2 (\mathbbm{1}_{\{p\geq i\}} \lambda^{(p)} - \lambda^{(i)}) t\right)} \to 0 \text{ as } t \to \infty
\end{equation*}
and we do indeed select the correct coefficient vector at an exponential rate. Note that in particular this equation for trajectory on the simplex is decoupled from the trajectory of the norm $m_i$ of the coefficients.

Of course, we are really interested in the case of in-exact parents. We can provide a heuristic argument similar to one of \cite{gemp2021}. Note that the updates for $w_i$ only depend on it's parents $w_{j<i}$, and one can show that an $\mathcal{O}(\epsilon)$ error in the parents propagates to an $\mathcal{O}(\epsilon)$ direction in the child gradient. We know $w_1$ will converge very fast to an arbitrary accuracy; then the gradient for $w_2$ will be very close to that corresponding to exact parents, so will quickly converge to that a similar order of accuracy; then the gradient of $w_3$ will be close to that for exact parents and so on. 

\subsubsection{Extending to stochastic case}
The real case of interest is the discrete time case with mini-batches. \cite{gemp2021} claim that their algorithm converges almost surely provided the step-size sequence $\eta_t$ satisfies
\begin{equation}
    \sum_{t=1}^\infty \eta_t = \infty, \qquad \sum_{t=1}^\infty \eta_t^2 < \infty
\end{equation}
Their key tool is a result on Stochastic Approximation (SA) on Riemannian manifolds \cite{shah_stochastic_2017}. This result extends the now-classical ODE method for analysis of SA schemes to Riemannian manifolds, mostly drawing on the presentation of \cite{borkar_stochastic_2008}. One key difficulty of applying the literature on SA schemes is obtaining stability bounds (saying that the estimates never get too big); this becomes trivial when considering updates on compact manifolds like the unit sphere, which is why \cite{gemp2021} are able to apply their SA tool `out-of-the-box'. In our case, because we do not restrict to the unit sphere, we are able to apply more classical results on SA, for example \cite{kushner_stochastic_2003}, however, we would need to prove the corresponding stability estimates. These should hold intuitively because the variance penalty term should keep estimates small, but they are technically difficult. We note that obtaining such stability estimates has attracted a lot of theoretical attention, but in practice they are often unnecessary because only a bounded subset of the parameter space is physically sensible. This applies to our GEP: it only makes sense to consider vectors $w$ with $w^T B w \leq 1$; and though $B$ is unknown in general, we may well be able to lower bound its eigenvalues, giving a bounded parameter space of interest. We could modify our algorithm to project onto this bounded subset of parameter space. The theory of \cite{kushner_stochastic_2003} can be applied to such a case of projected SA; indeed this theory has the added advantage of requiring weaker conditions on the step-sizes, namely only that
\begin{equation*}
    \sum_{t=1}^\infty \eta_t = \infty, \qquad \eta_t \to 0
\end{equation*}
Note such step-size schedules with slower decay is sometimes observed to give better empirical results in other SA problems \cite{kushner_stochastic_2003}.

We now point out what we understand to be a technical oversight in the proof of almost sure convergence in \cite{gemp2021}. They proof that $w_i$ converges a.s. to true value given fixed exact parents $w_{j<i}$ appears valid; as does the conclusion that $w_i$ converges a.s. to a corresponding optimum given fixed inexact parents; and also does their statement that if parents are close to correct then the corresponding optimum is close to correct. However, this does not say anything about convergence of $w_i$ when the parents are inexact and \textit{varying}; in particular the arguments of \cite{shah_stochastic_2017} do not apply in this case. We believe that it would be possible to fix this oversight by considering a suitable coupling of solution paths starting from an $\epsilon$-covering of a neighbourhood of the true solution. Alternatively, it may be possible to apply the result of \cite{shah_stochastic_2017} to the combined estimates $(w_1,\dots,w_k)$. Similar analysis will be needed for GEP-GHA because we also propose parallel updates for computational speed.

We next note that this SA literature also applies to our $\delta$-eigengame algorithm, whose updates are unbiased estimates to the gradient of the utilities $\mathcal{U}_i^\delta$. In this case, analysis may be more straightforward because we can also apply other existing literature on stochastic gradient descent.

We have not yet had time to make the discussion above more rigorous; we plan to do so in future work. Our algorithms fit very naturally into the well-studied SA framework, and we expect this literature to contain useful intuition and suggestions for implementation, as well as theoretical guarantees.

\section{Proof of subspace characterisations}\label{sec:subspace-charac}
Our main objective in this section is to prove \cref{prop:subspace-charac}, restated below.
\subspacecharac*

We write this more explicitly in Proposition \ref{prop:unconstr-charac}. In the rest of this section we adopt notation common in the optimisation literature. Let $\PSD{d},\PD{d}$ denote the space of $d \times d$ positive semi-definite and positive definite matrices respectively. 
\subsection{Background and statement}

We previously claimed that top-$k$ subspaces to a GEP $(A,B)$ can be characterised by
\begin{equation}
    \max_{W \in \R^{d\times k}} \tr(W^{\top} A W) \quad \text{ subject to }\:  W^{\top} B W = I_k
\end{equation}

It is accepted in the Machine Learning community that this is maximised precisely when $W$ is a top-$k$ subspace (\cite{ghojogh_eigenvalue_2022}, \cite{sameh_trace_1982}); however, we were unable to find a proof of this. In the matrix analysis literature, it is only usually established that
\begin{equation}\label{eq:subspace-gep-ineq}
    \max_{W \in \R^{d\times k}:W^{\top} B W = I_k } \tr(W^{\top} A W) \leq \sum_{i=1}^k \lambda_i = \tr({W^*}^\top A W^*)
\end{equation}
but there is no detail given for the equality case; for example see \cite{stewart_matrix_1990}, \cite{horn_matrix_1985}, \cite{tao_254a_2010}; to remedy this, we provide a self-contained proof to the result, which we state formally in Proposition \ref{prop:constr-charac} below. First, we explicitly define terms introduced in the main text.

\begin{definition}[Top-$k$ subspace]
    Let the GEP $(A,B)$ on $\R^d$ have eigenvalues $\lambda_1 \geq \dots \geq \lambda_d$. Then \textbf{a} top-$k$ subspace is that spanned by some $w_1,\dots,w_k$, where $w_i$ is a $\lambda_i$-eigenvector of $(A,B)$ for $i=1,\dots,k$.
\end{definition}

\begin{definition}[$B$-orthonormality]
    Let $B \in \PD{d}$ and $W\in \R^{d\times k}$. Then we say $W$ has $B$-orthonormal columns if $W^T B W = I_k$.
\end{definition}

\begin{restatable}[Constrained characterisation]{proprep}{constrainedcharac}
\label{prop:constr-charac}
    Let  $B \in \PD{d}$ and let $A \in \R^{d\times d}$ be symmetric with an orthonormal set of (generalised) eigenvectors $u_1,\dots u_d$ with (generalised) eigenvalues $\lambda_1 \geq \dots \lambda_d$. Consider $W \in \R^{d\times k}$ such that $W^T B W = I_k$.
    Then 
    \begin{equation}\label{eq:constrained-program}
        \tr(W^T A W) \leq \sum_{i=1}^k \lambda_i
    \end{equation}
    with equality if and only if $W$ defines a top-$k$ subspace for $(A,B)$.
\end{restatable}

We can now write \cref{prop:subspace-charac} more explicitly as
\begin{restatable}[Unconstrained characterisation]{proprep}{unconstrainedcharac}
\label{prop:unconstr-charac}
    Let $B \in \PD{d}$ and now let $A \in \PSD{d}$ \textbf{be positive semi-definite}. Then 
    \begin{equation}\label{eq:subspace-unconstrained-max-val}
    \max_{W \in \R^{d\times k}} \tr\left( W^{\top} A W \: (2 \, I_k - W^{\top} B W) \right)
    =
    \sum_{i=1}^k \lambda_i
    \end{equation}
    For any maximiser $W$, the columns of $W$ span a top-$k$ subspace for $(A,B)$. 
    Moreover, if the top-$k$ eigenvalues are positive (i.e. $\lambda_k>0$), then $W$ is a maximiser if and only if it has $B$-orthonormal columns spanning a top-$k$ subspace for $(A,B)$.
\end{restatable}

First a quick remark. When $W$ is $B$-orthonormal, the left hand side of (\ref{eq:subspace-unconstrained-max-val}) reduces to that of (\ref{eq:constrained-program}). Moreover, if the columns of $W$ are some (ordered) top-$k$ eigenvectors then $W^T A W = \operatorname{diag}(\lambda_1,\dots,\lambda_k)$ and so the maximum is certainly attained. The hard part is to characterise which $W$ achieve this maximum.

The rest of this section is as follows. In Section \ref{sec:reduce-to-identity} we show it is sufficient to prove Propositions \ref{prop:constr-charac} and \ref{prop:unconstr-charac} for the case where $B=I$. In Section \ref{sec:orthog-case} we give a self contained proof of Proposition \ref{prop:constr-charac}. In Section \ref{sec:reduce-to-orthog} we use this to give a self contained proof of Proposition \ref{prop:unconstr-charac}. Finally in Section \ref{sec:von-Neumann} we give an alternative (and shorter) proof of Proposition \ref{prop:unconstr-charac}  which we believe gives complementary intuition; this uses a powerful (but little known) trace inequality originally due to von Neumann.

\subsection{Reduction to \texorpdfstring{$B=I$}{B=I} cases}\label{sec:reduce-to-identity}
    In this section we show it is sufficient to prove Propositions \ref{prop:constr-charac} and \ref{prop:unconstr-charac} for the case where $B=I$.
    To start, consider an arbitrary $W\in \R^{d \times k}$. Then write $\tilde{W} = B^\half W$ and $\tilde{A} = B^\mhalf A B^\mhalf$. Note the map $W \to \tilde{W}$ is bijective because $B\in\PD{d}$. We will see a correspondence between the variational characterisations over $W$ for $(A,B)$ with those over $\tilde{W}$ for $(\tilde{A},I)$.
    
    Firstly, there is a correspondance of eigenvectors: it is well known that $\tilde{w_i}$ is a $\lambda_i$-eigenvector of $\tilde{A}$ precisely when $w_i \defeq B^\mhalf \tilde{w_i}$ is a $\lambda_i$-generalised eigenvector of $(A,B)$; indeed, if $\tilde{w}_i$ is a $\lambda_i$-eigenvector of $\tilde{A}$ then
    \begin{equation*}
        A (B^{\mhalf} \tilde{w}_i) 
        = B^{\half} B^{\mhalf} A B^{\mhalf} \tilde{w}_i
        = B^{\half} (\tilde{A} \tilde{w}_i)
        = B^{\half} (\lambda_i \tilde{w}_i) 
        = \lambda_i B (B^{\mhalf} \tilde{w}_i)
    \end{equation*}
    so $w_i \defeq B^{\mhalf} \tilde{w}_i$ is a $\lambda_i$-generalised eigenvector of $(A,B)$. The reverse implication is similar. Therefore we see $\tilde{W}$ is a top-$k$ subspace for $\tilde{A}$ if and only if $W$ is a top-$k$ subspace for $(A,B)$.

    We also see correspondence of the objective and orthogonality conditions via
    \begin{align*}
        W^T B W &= \tilde{W}^{\top} B^\mhalf B B^\mhalf \tilde{W} = \tilde{W}^T \tilde{W}\\
        W^{\top} A W &= \tilde{W}^{\top} B^\mhalf A B^\mhalf \tilde{W} = \tilde{W}^{\top} \tilde{A} \tilde{W}
    \end{align*}

    Finally note that $A$ is positive semi-definite if and only if $\tilde{A}$ is postive semi-definite.

    It is now easy to check that all conditions and results in Propositions \ref{prop:constr-charac} and \ref{prop:unconstr-charac} precisely map from statements about $W$ for the GEP $(A,B)$ to statements about $\tilde{W}$ for eigenvalue problem $(\tilde{A},I)$. In the remaining proofs we will drop the tilde's for readability.

\subsection{Proof of constrained characterisation}\label{sec:orthog-case}
We shall work via a series of Lemmas. These build up to Lemma \ref{lem:trace-obj}; the inequality case here corresponds to (\ref{eq:subspace-gep-ineq}) and is well known (\cite{stewart_matrix_1990}, \cite{horn_matrix_1985}, \cite{tao_254a_2010}); however, to analyse the equality case we needed a small extra argument; the key is  Lemma \ref{lem:haemers} which is a special case of Theorem 2.1 from \cite{haemers_interlacing_1995}; the corresponding proof makes extensive use of Raleigh's principle, Lemma \ref{lem:raleigh}. We adapt results to our notation and present all proofs here for completeness. Note that there is no restriction on the eigenvalues of $A$ to be non-negative in this sub-section.

\subsubsection{Results from \texorpdfstring{\cite{haemers_interlacing_1995}}{Haemers (1995)}}
\begin{lemma}[Rayleigh's principle]\label{lem:raleigh}
    Let $A \in \R^{d\times d}$ be symmetric with an orthonormal set of eigenvectors $u_1,\dots, u_d$ with eigenvalues $\lambda_1 \geq \dots \geq \lambda_d$.
    Then 
    \begin{equation*}
        \frac{u^T A u}{u^T u} \geq \lambda_i \text{ if } u \in \langle u_1,\dots u_i \rangle
        \quad \text{ and } \quad
        \frac{u^T A u}{u^T u} \leq \lambda_i \text{ if } u \in \langle u_1,\dots u_{i-1} \rangle^\perp
    \end{equation*}
    In either case, equality implies that $u$ is a $\lambda_i$ eigenvector of $A$.
\end{lemma}
\begin{proof}
    For the first case, suppose $u = \sum_{j=1}^i \nu_j u_j$ for some coefficients $\nu_j \in \R$. Then the Raleigh quotient
    \begin{equation*}
        \frac{u^T A u}{u^T u} = \frac{\sum_{j=1}^i \lambda_j \nu_j^2}{\sum_{j=1}^i \nu_j^2} \geq \frac{\lambda_i \sum_{j=1}^i \nu_j^2}{\sum_{j=1}^i \nu_j^2} = \lambda_i
    \end{equation*}
    with equality if and only if $\nu_j = 0$ whenever $\lambda_j > \lambda_i$; so indeed
    \begin{equation*}
        A u = \sum_{j=1}^i \lambda_j \nu_j u_j = \lambda_i \sum_{j=1}^i \nu_j u_j = \lambda_i u
    \end{equation*}
    Similarly, for the second case, we have $u = \sum_{j=i}^d \nu_j u_j$ so 
    \begin{equation*}
        \frac{u^T A u}{u^T u} = \frac{\sum_{j=i}^d \lambda_j \nu_j^2}{\sum_{j=i}^d \nu_j^2} \leq \frac{\lambda_i \sum_{j=i}^d \nu_j^2}{\sum_{j=i}^d \nu_j^2} = \lambda_i
    \end{equation*}
    with equality if and only if $\nu_j = 0$ whenever $\lambda_j < \lambda_i$; so indeed
    \begin{equation*}
        A u = \sum_{j=i}^d \lambda_j \nu_j u_j = \lambda_i \sum_{j=i}^d \nu_j u_j = \lambda_i u
    \end{equation*}

\end{proof}
\begin{lemma}\label{lem:haemers}
    Let $S \in \R^{d\times k}$ such that $S^T S = I$ and let $A \in \R^{d\times d}$ be symmetric with an orthonormal set of eigenvectors $u_1,\dots u_d$ with eigenvalues $\lambda_1 \geq \dots \lambda_d$.
    Define $C = S^T A S$, and let $C$ have eigenvalues $\mu_1 \geq \dots \geq \mu_k$ with respective eigenvectors $v_1 \dots v_k$.

    Then 
    \begin{itemize}
        \item $\mu_i \leq \lambda_i$ for $i=1,\dots,k$.
        \item if $\mu_i = \lambda_i$ for some $i$ then $C$ has a $\mu_i$-eigenvector $v$ such that $Sv$ is a $\mu_i$-eigenvector of $A$.
        \item if $\mu_i = \lambda_i$ for $i=1,\dots,k$ then $Sv_i$ is a $\mu_i$-eigenvector of $A$ for $i=1,\dots,k$.
    \end{itemize}
\end{lemma}

\begin{proof}
    With $u_1,\dots,u_d$ as above, for each $i \in [1,k]$, take a non-zero $s_i$ in
    \begin{equation}\label{eq:clever-subspace}
        \langle v_1,\dots, v_i \rangle \cap \langle S^T u_1,\dots,S^T u_{i-1}\rangle^\perp
    \end{equation}
    Then $S s_i\in \langle u_1\dots,u_{i-1}\rangle^\perp$, so by Raleigh's principle
    \begin{equation*}
        \lambda_i 
        \geq \frac{(S s_i)^T A (S s_i)}{(Ss_i)^T (S s_i)}
        =    \frac{s_i^T C s_i}{s_i^T s_i}
        \geq \mu_i.
    \end{equation*}

    The second bullet then follows from the equality case of Raleigh's principle: if $\lambda_i = \mu_i$ then $s_i$ and $S s_i$ are $\lambda_i$-eigenvectors of $C$ and $A$ respectively.

    For the final bullet we use induction on $l$: suppose $S v_i = u_i$ for $i=1,\dots,l-1$. Then we may take $s_l = v_l$ in (\ref{eq:clever-subspace}), but by the preceding argument, we saw $S v_l = Ss_l$ is a $\lambda_l$-eigenvector of $A$.
\end{proof}

The following lemma is now straightforward to prove.
\begin{lemma}\label{lem:trace-obj}
    Under the notation of Lemma \ref{lem:haemers}, we have
    \begin{equation*}
        \tr(S^T A S) \leq \sum_{i=1}^k \lambda_i
    \end{equation*}
    with equality if and only if $S$ defines a top-$k$ subspace for $A$.
\end{lemma}
\begin{proof}
    By Lemma \ref{lem:haemers}, letting $\mu_1 \geq \dots \geq \mu_k$ be the eigenvalues of $C = S^T A S$, we have
    \begin{equation*}
        \tr(S^T A S) = \sum_{i=1}^k \mu_i
        \geq 
        \lambda_1 + \sum_{i=2}^k \mu_i
        \geq 
        \dots
        \geq 
        \sum_{i=1}^k \lambda_i
    \end{equation*}
    with equality if and only if all the inequalities above are equalities, i.e. $\mu_i = \lambda_i$ for $i=1,\dots,k$. But then by Lemma \ref{lem:haemers}, we have that $Sv_i$ is a $\mu_i$-eigenvector of $A$.

    In particular, note that $V = \left(v_1,\dots,v_k\right)$ is orthogonal and $S V = \left(S v_1,\dots,S v_k\right)$ is a matrix whose columns are a choice of top-$k$ eigenvectors for $A$.
\end{proof}
\vspace{10pt}

From this Lemma we can immediately deduce Proposition \ref{prop:constr-charac} by the arguments of Section \ref{sec:reduce-to-identity}. We state this again now for reference.
\constrainedcharac*

\subsection{Proof of unconstrained characterisation}\label{sec:reduce-to-orthog}
    For this section we introduce the notation
    \begin{equation}\label{eq:h-notation}
        h(W;A,B) \defeq \tr\left( W^{\top} A W \: (2 \, I - W^{\top} B W) \right)
    \end{equation}
    Again, by the arguments of Section \ref{sec:reduce-to-identity} it would be sufficient to consider the case $B=I$; however, this does not save much work in this section, so we often leave the $B$ terms in to help intuition. The key idea here is that for any $W$ we can find a matrix with the same column space as $W$ and $(B-)$orthonormal columns which improves $h$; this then reduces the result to the constrained characterisation \ref{prop:constr-charac}. Again we work via a sequence of lemmas, progressively analysing $h$ under more general conditions. The hardest part is arguably the compactness argument in Lemma \ref{lem:matrix-optim}, which crucially requires $A$ to have strictly positive eigenvalues. Zero eigenvalues lead to some degeneracy and are dealt with carefully in Lemma \ref{lem:diag-case}. We first consider square $W \in \R^{k\times k}$ (Lemma \ref{lem:square-case}) to later apply to a basis for the column space of a `tall' $W \in \R^{d\times k}$ (Lemma \ref{lem:general-case}).
    
    \vspace{10pt}
    \begin{lemma}[Matrix perspective]\label{lem:matrix-optim}
        Let $\Lambda \in \R^{p\times p}$ be diagonal with strictly positive entries. Then for any $M \in \PSD{d}$ we have
        \begin{equation*}
            \mathbbm{h}(M) \defeq \tr\left(\Lambda M (2I - M) \right) \leq \tr(\Lambda)
        \end{equation*}
        with equality if and only if $M=I_p$.
    \end{lemma}
    \begin{proof}
        For an arbitrary $M \in \PSD{p}$, we can write $M = \sum_{i=1}^p \mu_i v_i v_i^T$ for some $\mu_1\geq\dots \geq \mu_p \geq 0$ and orthonormal set of vectors $(v_i)_{i=1}^p$. Now suppose that $|\mu_1 - 1|>0$.
        Then  \vspace{5pt}
        \begin{align*}
            \mathbbm{h}(M) 
            & \defeq \tr\left(\Lambda M (2I - M) \right) \\
            &= \tr\left(\Lambda \sum_{i=1}^p \mu_i(2 - \mu_i) v_i v_i^T \right) \\
            &= \sum_{i=1}^p \mu_i(2 - \mu_i) \tr\left(\Lambda v_i v_i^T \right) \\
            &= \sum_{i=1}^p \underbrace{\left(1-(\mu_i - 1)^2\right)}_{\leq 1} \underbrace{v_i^T \Lambda v_i}_{\in [\lambda_\text{min},\lambda_\text{max}]} \\
            &\leq \lambda_\text{min} \left(1-(\mu_1 - 1)^2\right) + (p-1)\lambda_\text{max}
        \end{align*}
        So in particular, if we have $|\mu_1| \geq 1 + \sqrt{1 + \frac{(p-1)\lambda_\text{max}}{\lambda_\text{min}}} \eqdef R^*$ then $(\mu_1 - 1)^2 \geq 1 + \frac{(p-1)\lambda_\text{max}}{\lambda_\text{min}} $ and so $\mathbbm{h}(M) \leq 0$.
    
        However, we have $\mathbbm{h}(I_p) = \tr(\Lambda) > 0$. So it is sufficient to show that $I_p$ is the unique maximiser on the compact set $\mathcal{S} = \{M \succcurlyeq 0 : \|M\|_{op} \leq R^*\}$.
    
        Now note that $\mathbbm{h}$ is a continuous function of $M$, so the maximum on $\mathcal{S}$ is attained (by compactness); yet this maximum cannot be on the boundary, because $\mathbbm{h}(M) \leq 0$ for such $M$. So any maximiser must be on the interior of $\mathcal{S}$. Moreover, the $\mathbbm{h}$ is differentiable, and so the derivative at such a maximiser must be zero.
    
        Indeed, we can explicitly compute the derivative. Consider a symmetric perturbation $H$. Note that $\tr ( \Lambda M H) = \tr(H^T M^T \Lambda^T) = \tr(\Lambda H M)$ by transpose and cyclic invariance of trace. So we can write:
        \begin{align*}
            \mathbbm{h}(M+H) - \mathbbm{h}(M)
            &= 
            \tr\left( \lambda (M + H) (2 I - M - H)\right) - \tr\left(\Lambda M (2I - M)\right) \\
            &= \tr \left( \Lambda ( H (2 I - M ) - M H - H^2) \right)\\
            &= \tr \left(\Lambda ( 2 H - HM - MH) \right) + o(H)\\
            &= 2 \tr \left(H^T \Lambda (I - M) \right) + o(H)
        \end{align*}
        from which we conclude that $M = I$ is the only stationary point. Therefore, it must be the unique maximiser.
    \end{proof}
    \vspace{10pt}

    \begin{lemma}[Diagonal case]\label{lem:diag-case}
        Let $\Lambda = \operatorname{diag}(\lambda_1,\dots,\lambda_p,0,\dots,0) \in \R^{k\times k}$ be diagonal with $\lambda_i > 0$ for $i=1,\dots,p$. Let $\Phi = \begin{pmatrix}\Phi_1 \\ \Phi_2 \end{pmatrix}\in \R^{k \times k}$ with $\Phi_1 \in \R^{p \times k}, \Phi_2 \in \R^{(k-p) \times k}$. Then 
        \begin{equation*}
            h(\Phi; \Lambda, I_k) \leq \sum_{i=1}^p \lambda_i = \tr(\Lambda)
        \end{equation*}
        with equality if and only if $\Phi_1 \Phi_1^T = I$ and $\Phi_1 \Phi_2^T = 0$.
    \end{lemma}
    
    \begin{proof}
        Write $\Lambda_1 = \operatorname{diag}(\lambda_1,\dots,\lambda_p)\in \R^{p\times p}$.
        Then 
        \begin{align*}
            h(\Phi;\Lambda,I) &= \tr \left(\begin{pmatrix}\Phi_1 \\ \Phi_2 \end{pmatrix}^T \begin{pmatrix}\Lambda_1 \Phi_1 \\ 0 \end{pmatrix} (2 I - \Phi_1^T \Phi_1 - \Phi_2^T \Phi_2) \right)\\[3pt]
            &=
            \tr \left( \Phi_1^T \Lambda_1 \Phi_1 (2 I - \Phi_1^T \Phi_1 - \Phi_2^T \Phi_2)\right) \\[3pt]
            &= \tr \left( \Phi_1^T \Lambda_1 \Phi_1 (2 I - \Phi_1^T \Phi_1) \right) - \tr \left( \Phi_1^T \Lambda_1 \Phi_1 \Phi_2^T \Phi_2\right) \\[3pt]
            &= \tr \left( \Phi_1^T \Lambda_1 \Phi_1 (2 I - \Phi_1^T \Phi_1) \right) - \tr((\Phi_1 \Phi_2^T)^T \Lambda_1 \Phi_1 \Phi_2^T) \\[3pt]
            &\leq \tr \left( \Phi_1^T \Lambda_1 \Phi_1 (2 I - \Phi_1^T \Phi_1) \right) \\[3pt]
            &= h(\begin{pmatrix}\Phi_1 \\ 0 \end{pmatrix};\Lambda, I)
        \end{align*}
        With equality if and only if $\Phi_1 \Phi_2^T = 0$ (because $\Lambda_1$ is strictly positive definite); i.e. the rows of $\Phi_1$ are orthogonal to the rows of $\Phi_2$.
        
        Now we show we can further improve $h$ by picking a $\tilde{\Phi}_1$ with orthogonal rows. 
        We have by cyclicity of trace and rearranging the brackets that
        \begin{align*}
        h(\begin{pmatrix}\Phi_1 \\ 0 \end{pmatrix};\Lambda, I) 
        &= 
        \tr \left( \Phi_1^T \Lambda_1 \Phi_1 (2 I - \Phi_1^T \Phi_1) \right) \\
        &=
        \tr \left( \Lambda_1 \Phi_1 \Phi_1^T (2 I -  \Phi_1 \Phi_1^T ) \right) \\[6pt]
        &\leq
        \tr \left( \Lambda_1\right)\\
        &=
        \tr\left(\Lambda\right)
        \end{align*}
        with equality if and only if $\Phi_1 \Phi_1^T = I$ by Lemma \ref{lem:matrix-optim}.
    
        Combining the two results we conclude the claim: $h(\Phi;\Lambda,I) \leq \tr(\Lambda)$
        with equality if and only if $\Phi_1 \Phi_1^T = I$ and $\Phi_1 \Phi_2^T = 0$.
        
        In addition, we note that the choice $\Phi = I$ attains the maximum, but there may be some extra flexibility corresponding to zero eigenvalues of $\Lambda$.    
    \end{proof}
    \vspace{10pt}

    \begin{lemma}[Square case]\label{lem:square-case}
        Let $A \in \PSD{k},\: B \in \PD{k}$ and $W \in \R^{k\times k}$. Then for any $B$-orthogonal matrix $\tilde{W}\in\R^{k\times k}$
        \begin{equation*}
            h(W;A,B) \leq h(\tilde{W};A,B) = \sum_{i=1}^k \lambda_i
        \end{equation*}
    \end{lemma}
    \begin{proof}
        Let $W^*$ be a matrix whose columns are successive generalised eigenvectors of the GEP $(A,B)$. Then ${W^*}^T B W^* = I_k$ and $(W^*)^T A W^* = \Lambda = \operatorname{diag}(\lambda_1,\dots,\lambda_k)$ by construction. Therefore $W^*$ is invertible and we can write $W = W^* \Phi$ for some $\Phi \in \R^{k\times k}$. Then
        \begin{align*}
            h(W; A,B)
            &= 
            \tr\left( \Phi^T (W^*)^T A W^* \Phi ( 2 I - \Phi^T (W^*)^T B W^* \Phi) \right) \\
            &=
            \tr\left( \Phi^T \Lambda \Phi (2 I - \Phi^T \Phi) \right) \\
            &=
            h(\Phi; \Lambda, I)\\
            &\leq
            \tr(\Lambda)
        \end{align*}
        where the inequality is from Lemma \ref{lem:diag-case}.
        Moreover, $W$ is $B$-orthogonal precisely when $\Phi$ is orthogonal, and so we have equality from Lemma \ref{lem:diag-case} in this case.
    \end{proof}
    \vspace{10pt}

\begin{lemma}[General case]\label{lem:general-case}
    Let $A \in \PSD{d},\: B \in \PD{d}$ and $W \in \R^{d\times k}$. Then there exists some $\bar{W} \in \R^{d\times k}$ with the same column space as $W$ such that $\bar{W}^T B \bar{W} = I_k$ and
    \begin{equation*}
        h(W;A,B) \leq h(\bar{W};A,B)
    \end{equation*}
\end{lemma}
\begin{proof}
    We have
    \begin{align*}
        h(W;A,B)
        &=
        h(I_k; W^T A W, W^T B W) \\
        &\leq
        h(\bar{\Gamma}; W^T A W, W^T B W) \\
        &=
        h(\bar{W}; A,B)
    \end{align*}
    by Lemma \ref{lem:square-case} where $\bar{\Gamma}$ is some $W^T B W$-orthogonal matrix, and $\bar{W} = \bar{\Gamma}$. Then $\bar{W}^T B \bar{W} = \bar{\Gamma}^T (W^T B W) \bar{\Gamma} = I_k$ as required.
\end{proof}

\subsubsection{Combining}\label{sec:combining}
    We now combine the results of the previous sections to prove our original proposition.
    \unconstrainedcharac*
    \begin{proof}
        \textit{of Proposition \ref{prop:unconstr-charac}:}
        Let $W \in \R^{d\times k}$. Then by Lemma \ref{lem:general-case}, there exists some $B$-orthogonal $\bar{W}$ with the same column space as $W$ and $h(W;A,B) \leq h(\bar{W};A,B)$. So combining with Proposition \ref{prop:constr-charac} we see
        \begin{equation}\label{eq:combining}
            h(W;A,B) \leq h(\bar{W};A,B) = \tr\left(\bar{W}^T A \bar{W}\right) \leq \sum_{i=1}^k \lambda_i
        \end{equation}
        where the second inequality is an equality if and only if $\bar{W}$ defines a top-$k$ subspace. This proves the first two claims.
    
        For the final claim, note $W$ is a maximiser if and only if both inequalities in (\ref{eq:combining}) are equalities. In this case, we can write $W = W^* \Phi$ for some $W^*\in\R^{d\times k}$ whose columns are some top-$k$ eigenvectors and $\Phi \in \R^{k\times k}$ (because $\bar{W}$ has the same column space as $W$ and defines a top-$k$ subspace).
        Then 
        \begin{equation*}
            \sum_{i=1}^k \lambda_i
            =
            h(W;A,B)
            =
            h(\Phi;\Lambda_k,I_k)
        \end{equation*}
        where $\Lambda_k = \operatorname{diag}(\lambda_1,\dots,\lambda_k)$. If $\lambda_k>0$ then $\Lambda_k\succ 0$ so by the equality case of Lemma \ref{lem:diag-case}, $\Phi$ must be orthogonal. So indeed $W$ has $B$-orthonormal columns spanning a top-$k$ subspace.
    \end{proof}

    Note that in the case where some of the top-$k$ eigenvalues are zero, the equality case of Lemma \ref{lem:diag-case} still gives some sort of orthogonality in the subspace of positive eigenvectors but allows degeneracy in the zero eigenspace (both arbitrary magnitudes and non-orthogonal directions).

\subsection{Alternative proof using von Neumann's trace inequality}\label{sec:von-Neumann}
\subsubsection{Trace inequality and other background}
Here we give a cleaner proof of Proposition \ref{prop:unconstr-charac} using the following result originally due to von Neumann; a simple proof (which also extends the result to infinite dimensional Hilbert spaces) and recent discussion is given in \cite{carlsson_von_2021}; we shall only state the matrix case for readability. First however, we need a new definition.

\begin{definition}[Sharing singular vectors]\label{def:share-sing-vectors}
    Let $a,b \in \mathbb{N}$ and $q = \min(a,b)$. We say two matrices $X,Y \in \R^{a \times b}$ share singular vectors if there exist some (orthonormal sets) $\{u_j\}_{j=1}^q$ and $\{v_j\}_{j=1}^q$ such that
    \begin{equation*}
        X = \sum_{j=1}^q \sigma_j(X) u_j v_j^T,\quad 
        Y = \sum_{j=1}^q \sigma_j(Y) u_j v_j^T
    \end{equation*}    
\end{definition}

\begin{theorem}[von Neumann's trace inequality]\label{thm:vN-trace}
    Let $a,b \in \mathbb{N}$ and $q = \min(a,b)$. Let $X,Y \in \R^{a \times b}$.
    Then
    \begin{equation}\label{eq:von-Neumann}
        \langle X, Y \rangle \leq \sum_{j=1}^q \sigma_j(X) \sigma_j(Y)
    \end{equation}
    with equality if and only if $X$ and $Y$ share singular vectors (in the sense of Definition \ref{def:share-sing-vectors}).
\end{theorem}

We stress that `sharing singular vectors' does not just say that $X,Y$ have the same singular vectors, but these are in the same order. Why do we need this non-standard definition of `sharing singular vectors'? This is just a convenient way to deal with the degeneracy associated to repeated singular values; indeed it is often more natural to think in terms of partial isometries and polar decompositions when proving such results; see \cite{carlsson_von_2021} for further discussion.

Next we give a corollary which is adapted to our case of interest. First this requires a simple lemma relating SVD to eigendecomposition. This is well known but we give a quick proof for completeness. It is necessary to assume $A$ is \textbf{positive semi-definite}, the result is false for general $A$.
\begin{lemma}\label{lem:svd2edcomp}
    Let $A \in \PSD{d}$ have some SVD $A = \sum_{i=1}^q \lambda_i u_i v_i^T$ with $\lambda_i > 0 \text{ for } i=1,\dots,q$. Then this in fact gives an eigendecomposition, i.e. $u_i = v_i$ is a $\lambda_i$-eigenvector of $A$ for each $i=1,\dots,d$.
\end{lemma}
\begin{proof}
    If $v$ is an eigenvector of $A$ then it is also an eigenvector of $A^2$. Further, the eigenvectors of $A$ span $\R^d$ so these provide an eigenbasis for $A^2$. Hence any eigenvector of $A^2$ is in fact an eigenvector of $A$. Next, by symmetry of $A$
    \begin{equation*}
        A^2 v_k = A^T A v_k = \sum_{j=1}^d \sum_{i=1}^d \lambda_j v_j u_j^T \lambda_i u_i v_i^T v_k
        =
        \sum_{i,j=1}^d \lambda_i \lambda_j v_j \underbrace{u_j^T u_i}_{\delta_{ji}} \underbrace{v_i^T v_k}_{\delta_{ik}}
        =
        \lambda_k^2 v_k
    \end{equation*}
    So by the observation above, $v_k$ must be a $\lambda_k$-eigenvector of $A$.
    But then
    \begin{equation*}
        \lambda_k v_k = A v_k = \sum_{j=1}^d \lambda_j u_j v_j^T v_k = \lambda_k u_k
    \end{equation*}
    so cancelling the $\lambda_k$'s gives $u_k=v_k$ as required.    
\end{proof}

To state the corollary we need a new definition analagous to Definition \ref{def:share-sing-vectors}. We stress again that the eigenvectors share the same ordering - so this is stronger than saying the matrices are simultaneously diagonalisable. The final conclusion of the corollary is very weak and just provides a clean consequence of the notion of sharing eigenvectors to our previous notion of top-$k$ subspaces.

\begin{definition}[Sharing eigenvectors]\label{def:share-eigenvectors}
    Let symmetric matrices $A,M \in \R^{d \times d}$ have eigenvalues $\lambda_1 \geq \dots \geq \lambda_d,\: \mu_1,\geq \dots \geq \mu_d$ respectively. We say $A,M$ share eigenvectors if there exist (an orthonormal set) $\{w_j\}_{j=1}^d$ such that
    \begin{equation*}
        A = \sum_{j=1}^d \lambda_j w_j w_j^T,\quad 
        M = \sum_{j=1}^d \mu_j w_j w_j^T
    \end{equation*}    
\end{definition}

\begin{corollary}[Positive-definite von Neumann]\label{cor:psd-vN}
    Let $A \in \PSD{d}$ and $M \in \R^{d\times d}$ be symmetric with eigenvalues $(\lambda_i)_{i=1}^d, (\mu_i)_{i=1}^d$ respectively. 
    Then
    \begin{equation}
        \langle A, M \rangle \leq \sum_{i=1}^d \lambda_i \mu_i
    \end{equation}
    with equality if and only if $A$ and $M$ share eigenvectors.
    Equality therefore implies that any unique top-$k$ subspace for $A$ is a top-$k$ subspace for $M$ and vice versa.
\end{corollary}
\begin{proof}
    Write $M^+ = M - \mu_d I_d$. Then $M^+ \in \PSD{d}$.
    Note that for positive semidefinite matrices, their eigendecomposition provides a valid singular value decomposition; therefore (by uniqueness of SVD) the singular values of $A,M^+$ respectively are precisely $(\lambda_i)_{i=1}^d, (\mu_i-\mu_d)_{i=1}^d$. Applying von Neumann's inequality to $A,M^+$ gives
    \begin{equation*}
        \langle A , M \rangle
        =     \langle A, M^+ \rangle + \mu_d \langle A , I_d \rangle
        \leq  \left(\sum_{i=1}^d \lambda_i (\mu_i - \mu_d)\right) + \mu_d \sum_{i=1}^d \lambda_i
        =    \sum_{i=1}^d \lambda_i \mu_i
    \end{equation*}
    and so there exists some shared set of singular vectors $(u_i,v_i)_{i=1}^d$ for $(A,M^+)$.

    Now applying Lemma \ref{lem:svd2edcomp} we conclude that in fact the $(u_i)_{i=1}^d$ give a shared eigenvector basis for $(A,M^+)$, and therefore also for $(A,M)$.

    For the final conclusion, let the shared eigenvectors be $(u_i)_{i=1}^d$. Then $\operatorname{span}(u_1,\dots,u_k)$ is a top-$k$ subspace for $A$, so the unique such subspace must be this, which is also a top-$k$ subspace for $M$ (by definition).
\end{proof}

\subsubsection{Proof of variational characterisation}
We can now use this to prove Proposition \ref{prop:unconstr-charac}.
\begin{proof}[Proof of Proposition \ref{prop:unconstr-charac}]
    Following the discussion in Section \ref{sec:reduce-to-identity}, we may assume that $B=I$.
    then by cyclicity of trace
    \begin{align*}
        \tr\left( W^{\top} A W \: (2 \, I_k - W^{\top} W) \right)
        &=
        \tr\Bigl( A \: \underbrace{\{ WW^T \: (2 \, I_d - W W^{\top})\}}_{\defeq M} \Bigr)
    \end{align*}
    Note that $M$ defined above is of rank at most $k$ with eigenvalues $\leq 1$. Indeed, if we have the singular value decomposition $W = U D V^T$ then $W W^T = U D V^T V D U^T = U D^2 U^T$ and so
    \begin{equation*}
        M = U D^2 U^T (2 \, I_k - U D^2 U^T) = U ( 2 D^2 - D^4 ) U^T
    \end{equation*}
    so if $D$ has diagonal elements $(d_i)_{i=1}^k$ then the eigenvalues of $M$ are just
    \begin{equation*}
        \mu_i \defeq 2 d_i^2 - d_i^4 = 1 - (1-d_i^2)^2 \leq 1
    \end{equation*}
    for $i=1,\dots,k$ (the remaining $d-k$ eigenvalues are zeros, corresponding to the null space $U^\perp$).
    Without loss of generality, suppose we have ordered these elements such that $\mu_1 \geq \dots \geq \mu_k$.
    Then 
    \begin{equation}\label{eq:inequality-string}
        h(W;A,I) 
        =
        \tr(A M)
        \leq
        \sum_{i=1}^d \lambda_i \mu_i
        =
        \sum_{i=1}^k \lambda_i \mu_i
        \leq
        \sum_{i=1}^k \lambda_i
    \end{equation}
    where the first inequality is our positive definite version of von Neumann trace inequality (Corollary \ref{cor:psd-vN}), the central equality uses $\mu_{k+1}=\dots=\mu_d=0$, and the final inequality uses $\mu_i \leq 1$. 

    We now move on to examine the case of equality.
    \begin{itemize}
        \item If $\lambda_k > 0$ then equality holds for the final inequality if and only if $\mu_i = 1$ for $i=1,\dots,k$. But then the equality case of Corollary \ref{cor:psd-vN}, tells us that this unique top-$k$ subspace for $M$ must be some top-$k$ subspace for $A$.
        Next note $\mu_i =1 \implies d_i^2 = 1$ for all $i=1,\dots,k$ so $D^2=I$ (and $D$ is orthogonal). Thus $W^T W = V D U^T U D V^T = V D^2 V^T = V V^T = I$ so $W$ has orthonormal columns.
        \item The case where $\lambda_k = 0$ requires more care. Suppose $A$ has $q<k$ nonzero eigenvalues (i.e. $\lambda_q > 0, \lambda_{q+1} < 0$).
    Then replacing the end of (\ref{eq:inequality-string}) with
    \begin{equation*}
        \sum_{i=1}^k \lambda_i \mu_i = \sum_{i=1}^q \lambda_i \mu_i \leq \sum_{i=1}^q \lambda_i = \sum_{i=1}^k \lambda_i
    \end{equation*}
    we see that equality holds if and only if $\mu_i = 1$ for $i=1,\dots,q$ (note the remaining $\mu_i$'s are no longer constrained).
    Now consider the (unique) top-$q$ subspace of $A$ corresponding to non-zero eigenvalues. This must also be a top-$q$ subspace of $M$ (by Corollary \ref{cor:psd-vN}) and so all must be eigenvectors of $M$ with eigenvalue 1.
    In particular this top-$q$ subspace must be contained within the subspace spanned by the columns of $U$ corresponding to diagonal elements $d_i \in \{+1,-1\}$, so is certainly also contained within the column space of $W$. But since all the remaining eigenvalues are precisely zero, we see the column space of $W$ is indeed a top-$k$ subspace for $A$.  
    \end{itemize}
\end{proof}

\subsection{Discussion}

We now offer some intuition: a key strength of this unconstrained formulation is that it is straightforward to transform with respect to arbitrary changes of basis; therefore one can do analysis in the basis of generalised eigenvectors. By contrast, the orthogonality constraint in (\ref{eq:subspace-gep}) only permits orthogonal changes of basis. This may give intuition to why $\mu$-eigengame only works in the $B=I$ case but our approach is effective for general GEPs.

\section{Further connections to Previous Work}
\subsection{Relationship to previous DCCA formulations}\label{sec:dcca-alt-deriv}
In \cite{wang2015stochastic}, DCCA is formalised as
\begin{equation*}
    \max_{f,g, U \in \R^{d_x \times k},V \in \R^{d_y \times k}} \tr ( U^{\top} F G^{\top} V ) \quad \text{ subject to } U^{\top} F F^{\top} U = V^{\top} G G^{\top} V = I_k
\end{equation*}
where $F = \left(f(x_1),\dots, f(x_N)\right),G = \left( g(y_1),\dots,g(y_n)\right)$ are matrices whose columns are images of the training data under functions $f,g$ defined by neural networks in some class of functions $\mathcal{F},\mathcal{G}$ with input and output dimensions $(p,d_x),(q,d_y)$ respectively. Observe that this optimisation is really targeting the population problem
\begin{equation*}
    \max_{f\in\mathcal{F},g\in\mathcal{G}, U \in \R^{d_x \times k},V \in \R^{d_y \times k}} \tr ( U^{\top} \Sigma_{fg} V ) \quad \text{ subject to } U^{\top} \Sigma_{ff} U = V^{\top} \Sigma_{gg}^{\top} V = I_k
\end{equation*}
Where $\Sigma_{ff} = \operatorname{Var}(f(X)),\Sigma_{f,g} = \operatorname{Cov}(f(X),g(Y)),\Sigma_{gg}=\operatorname{Var}(g(Y))$. 
We now abuse notation to write $R_k(f(X),g(Y))$ to correspond to the sum of the first $k$ canonical correlations of the pair of random variables $f(X),g(Y)$.
It is well known that if we fix $f,g$ in the above then the optimisation problem defines the top-$k$ subspace for CCA.
So we can write the optimisation as
\begin{equation*}
    \max_{f\in\mathcal{F},g\in\mathcal{G}} R_k(f(X),g(X))
\end{equation*}
But then using \cref{prop:subspace-charac}, this optimisation is also equivalent to
\begin{equation}\label{eq:dcca-our-subspace-charac}
    \max_{f\in\mathcal{F},g\in\mathcal{G},W \in \R^{(d_x+d_y)\times k}} \tr\left( W^{\top} A_{fg} W \: (2 \, I_k - W^{\top} B_{fg} W) \right)
\end{equation}
where we define
\begin{equation*}
    A_{fg} = \begin{pmatrix}0 &\Sigma_{fg} \\ \Sigma_{gf} & 0\end{pmatrix}, \qquad
	B_{fg} = \begin{pmatrix}\Sigma_{ff} & 0 \\ 0 & \Sigma_{gg}\end{pmatrix}, \qquad
	d=p+q.
\end{equation*}
We have now almost recovered the form of (\ref{eq:dcca-usum}), the only difference is that there is an optimisation over W in the above. To finish the derivation we follow \cite{wang2015stochastic} and define the augmented function classes:
\begin{equation*}
    \tilde{\mathcal{F}} = \{\tilde{f} = U^{\top} f : f \in \mathcal{F}, U\in \R^{d_x \times k}\}, \quad
    \tilde{\mathcal{G}} = \{\tilde{g} = V^{\top} g : g \in \mathcal{G}, V\in \R^{d_y \times k}\}
\end{equation*}
For any $W \in \R^{(d_x+d_y)\times k}$ there exist unique $U\in \R^{d_x \times k}, V\in \R^{d_y \times k}$ with $W^{\top} = (U^{\top}, V^{\top})$. Also for $\tilde{f} = U^{\top} f, \tilde{g} = V^{\top} g$ it follows from definition of covariance that
\begin{equation*}
    U^{\top} \Sigma_{ff} U = \Sigma_{\tilde{f}\tilde{f}}, \quad
    U^{\top} \Sigma_{fg} U = \Sigma_{\tilde{f}\tilde{g}}, \quad
    V^{\top} \Sigma_{gg} V = \Sigma_{\tilde{g}\tilde{g}}
\end{equation*}
so indeed we can write (\ref{eq:dcca-our-subspace-charac}) as
\begin{equation}
    \max_{f\in\tilde{\mathcal{F}},g\in\tilde{\mathcal{G}}} \tr\left( A_{\tilde{f}\tilde{g}} \: (2 \, I_k - B_{\tilde{f}\tilde{g}} ) \right)
\end{equation}
which precisely matches our objective in (\ref{eq:dcca-usum}).

We now comment on this analysis: the definition from \cite{wang2015stochastic} proposes DCCA to find a pair of low-dimensional feature maps under which the two sets of data are highly correlated. Intuitively, this analysis says that if we take a sufficiently expressive class of neural networks, we only need to consider a $k$ dimensional latent space to recover the top-$k$ subspace of `deep canonical directions'. Note also that one only needs to apply a $k$-dimensional classical CCA to recover the top-$k$ directions from this subspace. Finally, we warn that in general these directions may be highly non-unique, and that many of the nice properties of CCA are dependent on the structure of Euclidean space and do not hold for DCCA.

\subsection{Derivation of SGHA algorithm \texorpdfstring{from \cite{chen2019constrained}}{Chen 2019}}\label{sec:chen}

The Lagrangian function in \citet{chen2019constrained} corresponding to (\ref{eq:subspace-gep}) is given as:

\begin{align}\label{eq:sghqlagrangeutil}
 \mathcal{L}(W, \textcolor{blue}{\Gamma})&=\operatorname{tr}\left(W^{\top} A W\right)-\left\langle \textcolor{blue}{\Gamma}, W^{\top} B W-I\right\rangle
\end{align}
Differentiating with respect to $W$ gives
\begin{equation*}
    2 A W - 2 B W \Gamma = 0
\end{equation*}
Left multiplying by $W^T$ and using the constraint $W^T B W = I_k$ shows that at any stationary point we have
\begin{equation}\label{eq:lagr-at-optimum}
    \Gamma = W^T A W
\end{equation}
They then plug this value of $\Gamma$ into a gradient descent step for $W$ to obtain an update direction:
\begin{equation*}
    \Delta^\text{SGHA}(W) = - A W - B W (W^T A W)
\end{equation*}
(where we follow their exposition and drop the factor of 2 at this point). Note that this technique of plugging in the optimal dual variable is non-standard to our knowledge. The algorithm needs their theoretical results for more concrete justification.

\subsection{Pseudo-Utilities in Previous Work}\label{sec:pseudo-utils}

Recall we defined the pseudo-utility of GHA-GEP to be

\begin{align}
\mathcal{PU}_{i}^{\text{GHA-GEP}}(w_i | w_{j<i}, \bGam ) 
&=\hat{w}_{i}^{\top}A\hat{w}_{i}
+\textcolor{blue}{\Gamma_{ii}} (1 - \hat{w}_{i}^{\top}B\hat{w}_{i})
-2\sum_{j< i} \textcolor{blue}{\Gamma_{ij}} \hat{w}_{j}^{\top}B\hat{w}_{i}
\end{align}

is closely related to expressions in previous work. Consider first SGHA. The Lagrangian function in \citet{chen2019constrained} is given above in (\ref{eq:sghqlagrangeutil}). 

Considering a single `player' this utility can be written:

\begin{align}
\mathcal{PU}_{i}^{\text{SGHA}}(w_i , \bGam ) 
&=\hat{w}_{i}^{\top}A\hat{w}_{i}
+\textcolor{blue}{\Gamma_{ii}} (1 - \hat{w}_{i}^{\top}B\hat{w}_{i})
-\sum_{j} \textcolor{blue}{\Gamma_{ij}} \hat{w}_{j}^{\top}B\hat{w}_{i}
\end{align}

We can also write the updates of $\mu$-EigenGame \cite{gemp2021} as a Lagrangian pseudo-utility. Note that this is a slightly different expression to that given by the authors.

\begin{align}
\mathcal{PU}_{i}^{\mu}(w_i | w_{j<i}, \bGam) 
&=\hat{w}_{i}^{\top}A\hat{w}_{i}
-\sum_{j<i} \textcolor{blue}{\Gamma_{ij}} \hat{w}_{j}^{\top}\hat{w}_{i}
\end{align}

\subsection{Utility Shape}\label{sec:utilityshape}

\begin{lemma}
Let $\hat{w}_i=m(\cos \left(\theta_i\right) w_i+\sin \left(\theta_i\right) \Delta_i)$ where $\hat{w}_{i}^{\top} B\hat{w}_{i} = m$, then:

\begin{align}\label{eq:sinusoidagain}
\mathcal{U}_i\left(\hat{w}_i,w_{j<i}\right)=\mathcal{U}_i\left(mw_i,w_{j<i}\right)-\sin^2(\theta_i)(\mathcal{U}_i\left(mw_i,w_{j<i}\right)-\mathcal{U}_i\left(m\Delta_i,w_{j<i}\right))
\end{align}
\end{lemma}

We will show that this result follows from similar logic to \citet{gemp20} once the scaling factor $m$ is accounted for.

\begin{proof}
Let $\Delta_i=\sum_{l=1}^d p_l w_l,\|p\|=1$. Decomposing the utility function for player $i$ we have:

\begin{align}
\mathcal{U}_i\left(\hat{w}_i,w_{j<i}\right)&=2\langle\hat{w}_{i},A\hat{w}_{i}\rangle-\langle\hat{w}_{i},B\hat{w}_{i}\rangle\langle\hat{w}_{i},A\hat{w}_{i}\rangle-2\sum_{j< i}\langle\hat{w}_{i},Bw^{(j)}\rangle\langle w^{(j)},A\hat{w}_{i}\rangle\\
&=(2m-m^2)(\cos^2(\theta_i)\lambda_{ii} + \sin^2(\theta_i)\langle\Delta_i,\lambda\Delta_i\rangle) \notag\\&- 2m\sum_{j< i}\left\langle\cos \left(\theta_i\right) w_i+\sin \left(\theta_i\right) \Delta_i, Aw_j\right\rangle\left\langle\cos \left(\theta_i\right) w_i+\sin \left(\theta_i\right) \Delta_i, Bw_j\right\rangle\\
&=(2m-m^2)(\cos^2(\theta_i)\lambda_{ii} + \sin^2(\theta_i)\langle\Delta_i,\lambda\Delta_i\rangle) \notag\\&- 2m\sum_{j< i}\sin ^2(\theta_i)\left\langle\Delta_i, Aw_j\right\rangle\left\langle\Delta_i, Bw_j\right\rangle\\
&=(2m-m^2)\lambda_{ii}\notag\\ &- (2m-m^2)\sin^2(\theta_i) \lambda_{ii} + \sin^2(\theta_i) ((2m-m^2)\langle\Delta_i,A\Delta_i\rangle\notag\\&-2m\sum_{j< i}\left\langle\Delta_i, Aw_j\right\rangle\left\langle\Delta_i, Bw_j\right\rangle\\
&=(2m-m^2)\lambda_{ii}\notag\\&-\sin^2(\theta_i)((2m-m^2)\lambda_{ii}+(2m-m^2)\langle\Delta_i,A\Delta_i\rangle-2m\sum_{j< i}\left\langle\Delta_i, Aw_j\right\rangle\left\langle\Delta_i, Bw_j\right\rangle)\\
&=u_i\left(mw_i,w_{j<i}\right)-\sin^2(\theta_i)(\mathcal{U}_i\left(mw_i,w_{j<i}\right)-\mathcal{U}_i\left(m\Delta_i,w_{j<i}\right))
\end{align}

\end{proof}

\subsection{Utility as a projection deflation}\label{sec:defl}
\begin{align}\label{eq:projectiondeflation}
\mathcal{U}_{i}^{\delta}&=2\hat{w}_{i}^{\top}[\overbrace{I-\sum_{j< i} B \hat{w}_{j} \hat{w}_{j}^{\top}}^{\text {projection deflation}}] A \hat{w}_{i} - \hat{w}_{i}^{\top}B\hat{w}_{i}\hat{w}_{i}^{\top} A\hat{w}_{i}
\end{align}

Analogously to the previous work in $\alpha$-EigenGame \citep{gemp20}, the matrix $[I-\sum_{j\leq i} B \hat{w}_{j} \hat{w}_{j}^{\top}]$ has a natural interpretation as a projection deflation.

\section{\texorpdfstring{Vectorized $\delta$-EigenGame}{Vectorized delta EigenGame}}\label{sec:vectorized}

\begin{algorithm}[H]
   \caption{Vectorized $\delta$-EigenGame}
   \label{alg:deltaEigenGamevec}
\begin{algorithmic}
   \STATE {\bfseries Input:} data stream $Z_t$ consisting of $b$ samples from $z_n$, learning rate $\eta$
   \FOR{$t=1$ {\bfseries to} $T$}
    \STATE Construct independent unbiased estimates $\hat{A}$ and $\hat{B}$ from $Z_t$
    \STATE Rewards $\leftarrow 2A\hat{W}$
    \STATE Penalties $\leftarrow B\hat{W}\operatorname{triu}(\hat{W}^{\top}A\hat{W})$
    \STATE $\tilde{\nabla} \leftarrow$ Rewards - Penalties
    \STATE $\hat{W}^{\prime} \leftarrow \hat{W}+\eta_{t} \tilde{\nabla}$
   \ENDFOR
\end{algorithmic}
\end{algorithm}

Where $\operatorname{triu}$ returns a matrix with the entries below the main diagonal set to zero.

\section{A connection to self-supervised learning methods}\label{sec:ssl}

 Reorganizing our update equation (\ref{eq:ourupdate}), we find that the intuition of our method can also be understood as three terms encouraging variance in $A$, penalizing variance in $B$, and discouraging covariance.

\begin{align}
\Delta_{i}^{\delta}=\overbrace{A \hat{w}_{i}}^{\text{Reward}} -\overbrace{B\hat{w}_{i}\left(\hat{w}_{i}^{\top} A \hat{w}_{i}\right)}^{\text{Variance Penalty}} - \overbrace{\sum_{j < i}B\hat{w}_{j}\left(\hat{w}_{j}^{\top} A \hat{w}_{i}\right)}^{\text{Orthogonality Penalty}} \label{eq:ourupdatessl}
\end{align}

The motivation is similar to that in recent work in self-supervised learning \citep{zbontar2021barlow} and in particular the VICReg method in \citet{bardes2021vicreg}. Recent work has shown links between several self-supervised learning approaches and classical spectral embedding methods \citep{balestriero2022contrastive}, some of which could be represented by GEPs. Like CCA, many self-supervised learning approaches are based on finding a function which is invariant to an image and its augmented version i.e. the learnt representations of both are correlated. 

\section{Experiment Details}

\subsection{Total Correlation Captured (TCC)}\label{sec:tccdef}

This is the sum of the canonical correlations of the learnt representation (i.e. the sum of the top-k canonical correlations of $X$ and $Y$).

\subsection{Proportion of Correlation Captured (PCC)}\label{sec:pccdef}

This is the sum of the canonical correlations of the learnt representation as a proportion of the sum of the canonical correlations of the learnt representation using the population ground truth (i.e. the sum of the top-k canonical correlations of $X$ and $Y$).

\subsection{Stochastic CCA}

The latter two datasets are formed from left and right halves of the canonical datasets \citep{lecun2010mnist,krizhevsky2009learning}. With the same initialization for all methods, we trained for 10 epochs on each dataset with a mini-batch size of 128 and illustrate the models with the best performance in the validation set.

\subsection{Stochastic CCA Hyperparameters}\label{sec:hparams}

Learning rate was tuned from $\eta=(10^{-1},10^{-2},10^{-3},10^{-4},10^{-5})$ and $\gamma$-EigenGame parameter $\gamma$ was tuned from the same range. We used Jax \citep{deepmind2020jax} to optimize the linear CCA models using the Jaxline framework. We used WandB \citep{wandb} for experiment tracking to develop insights for this paper.

\subsection{Deep CCA}

We use two variants of paired MNIST datasets. The first is identical to the split MNIST dataset in the previous section. The second harder problem is closely related to an experiment by \citet{wang2015deep}. Their `noisy' paired MNIST data takes two different digits with the same class. The first is rotated randomly while the second has additive gaussian noise. Finally, we use the X-Ray Microbeam (XRMB) dataset from \citet{arora2016stochastic}. For all of the datasets, we use two encoders with 50 latent dimensions and two hidden layers with size 800 and leaky ReLu activation functions, a similar architecture to that used in \citet{wang2015stochastic}. We compare our proposed method to DCCA-NOI at mini-batch sizes of $20$ and $100$ and DCCA-STOL with mini-batch size $100$ and $500$ (DCCA-STOL cannot be used for mini-batch sizes less than the number of latent dimensions).

\subsection{DCCA Hyperparameters}

We trained for 30 epochs on each dataset with mini-batch sizes of 20 and 100. Learning rate was tuned from $\eta=(10^{-1},10^{-2},10^{-3},10^{-4},10^{-5})$ and the DCCA-NOI parameter $\rho$ was tuned between 0 and 1. We use PyTorch \citep{NEURIPS2019_9015} with the Adam optimizer \citep{kingma2014adam}.

\section{Additional Experiments}

\subsection{Stochastic CCA with smaller mini-batch sizes}\label{sec:ccaextra}

In this section we repeat the experiments described in the main text with smaller mini-batch sizes (64 and 32).

Results for mini-batch sizes 32 and 64 are broadly similar to those in the main text for mini-batch size 128. In the MNIST data we can see again that there is a tradeoff between speed of convergence in early iterations and the quality of the solution.

\begin{figure}[H]
     \centering
     \begin{subfigure}[b]{0.31\textwidth}
         \centering
         \includegraphics[width=\textwidth]{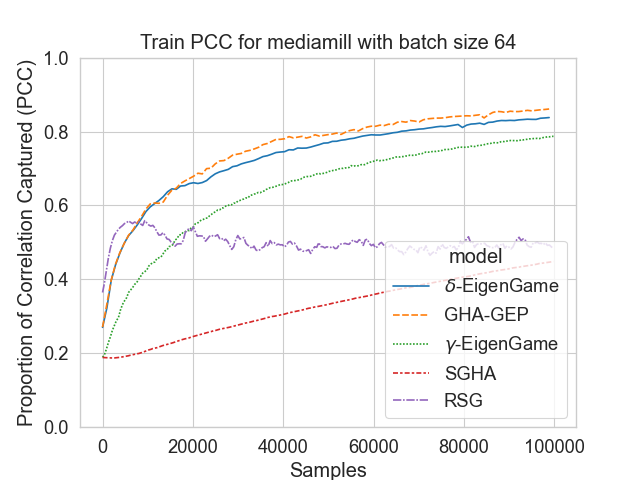}
         \label{fig:ccamediamill64}
     \end{subfigure}
     \hfill
     \begin{subfigure}[b]{0.31\textwidth}
         \centering
         \includegraphics[width=\textwidth]{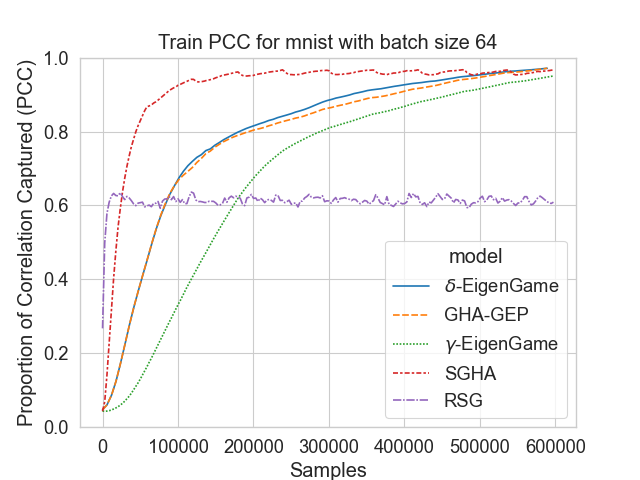}
         \label{fig:ccamnist64}
     \end{subfigure}
     \hfill
     \begin{subfigure}[b]{0.31\textwidth}
         \centering
         \includegraphics[width=\textwidth]{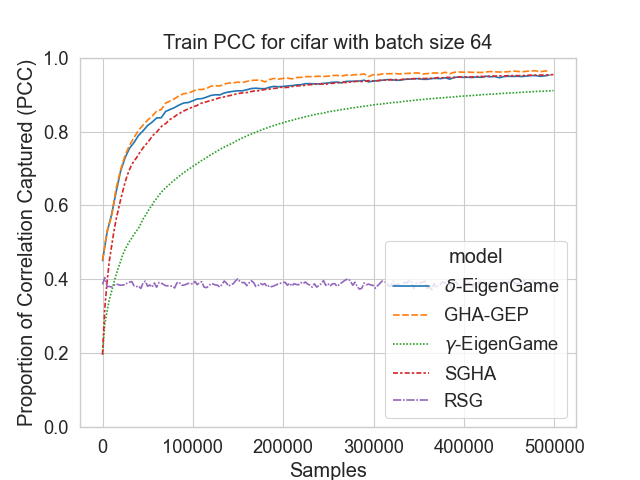}
         \label{fig:ccacifar64}
     \end{subfigure}
     \newline
     \centering
     \begin{subfigure}[b]{0.31\textwidth}
         \centering
         \includegraphics[width=\textwidth]{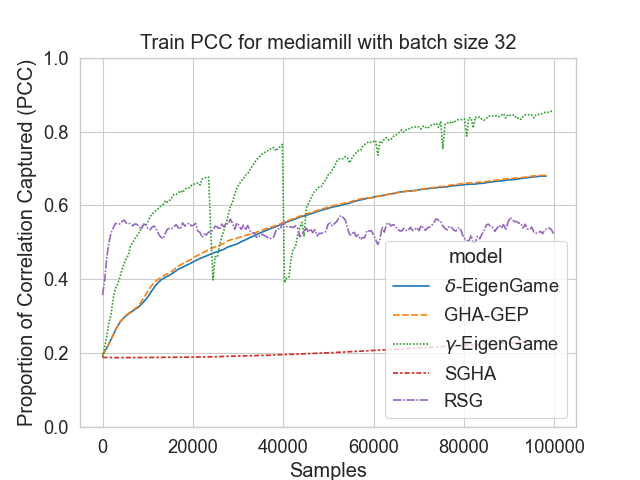}
         \label{fig:ccamediamill32}
     \end{subfigure}
     \hfill
     \begin{subfigure}[b]{0.31\textwidth}
         \centering
         \includegraphics[width=\textwidth]{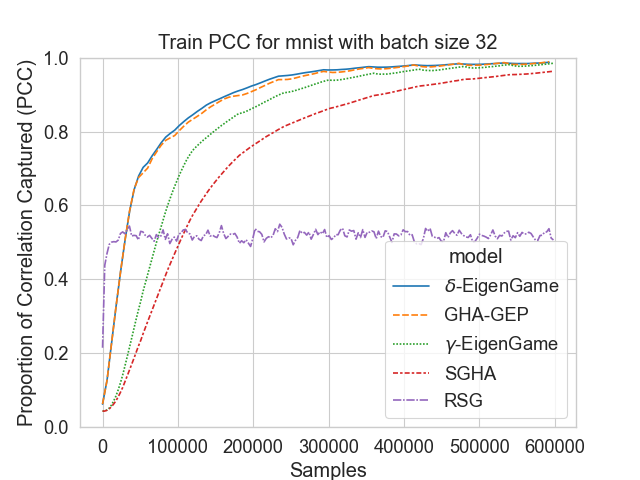}
         \label{fig:ccamnist32}
     \end{subfigure}
     \hfill
     \begin{subfigure}[b]{0.31\textwidth}
         \centering
         \includegraphics[width=\textwidth]{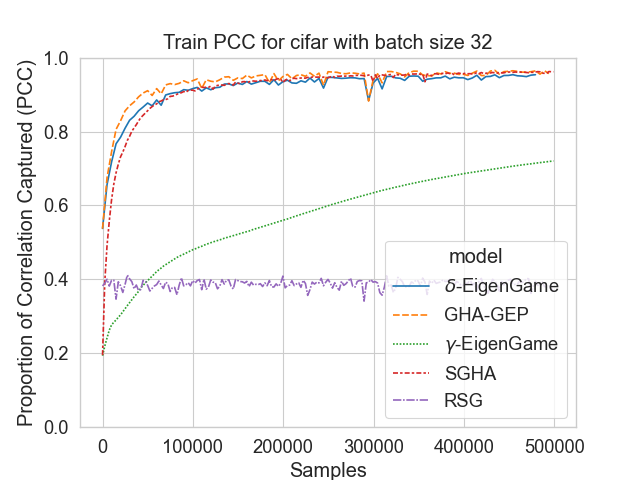}
         \label{fig:ccacifar32}
     \end{subfigure}
        \caption{CCA with stochastic mini-batches of size 64 (top) and 32 (bottom): proportion of correlation captured with respect to Scipy ground truth by $\delta$-EigenGame vs prior work. The maximum value is 1.}
        \label{fig:ccapcc32}
\end{figure}

\subsection{Partial Least Squares}

The Partial Least Squares (PLS) \citet{wold1984collinearity} problem can also be formulated as a similar GEP but with $B$ replaced by the identity matrix. PLS is equivalent to finding the singular value decomposition (SVD) of the covariance matrix $X^{\top}Y$. It has an interpretation as a (infinitely) ridge regularised CCA where the covariance matrices $\Sigma_{XX}$ and $\Sigma_{YY}$ are replaced by identity matrices; this corresponds to assuming no collinearity between variables.

\subsubsection{PLS with stochastic mini-batches}

In this experiment we compare our method to the Stochastic Power method \cite{arora2016stochastic}, $\gamma$-EigenGame, and SGHA for the stochastic PLS problem.

For these experiments we use the Proportion of Variance captured (PV). This is the sum of the singular values of the learnt representation using each stochastic optimisation method as a proportion of the sum of the singular values of the learnt representation using the population ground truth (i.e. the sum of the top-k singular values of the covariance matrix $X^{\top}Y$).

Figure \ref{fig:plspv} shows that all of the methods perform similarly in terms of variance captured across the datasets. While the stochastic power method is very fast to converge in the MNIST and CIFAR data, it solutions can be suboptimal. The performance of $\delta$-EigenGame is arguably more suprising for the PLS problem because both the Stochastic Power method and $\gamma$-EigenGame explicitly enforce the constraints  at each iteration whereas $\delta$-EigenGame only enforces the constraint via penalty terms. 

\begin{figure}[H]
     \centering
     \begin{subfigure}[b]{0.32\textwidth}
         \centering
         \includegraphics[width=\textwidth]{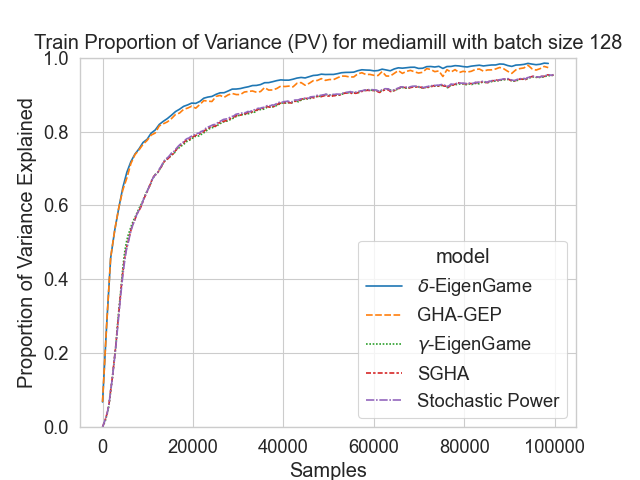}
         \label{fig:plsmediamill}
     \end{subfigure}
     \hfill
     \begin{subfigure}[b]{0.32\textwidth}
         \centering
         \includegraphics[width=\textwidth]{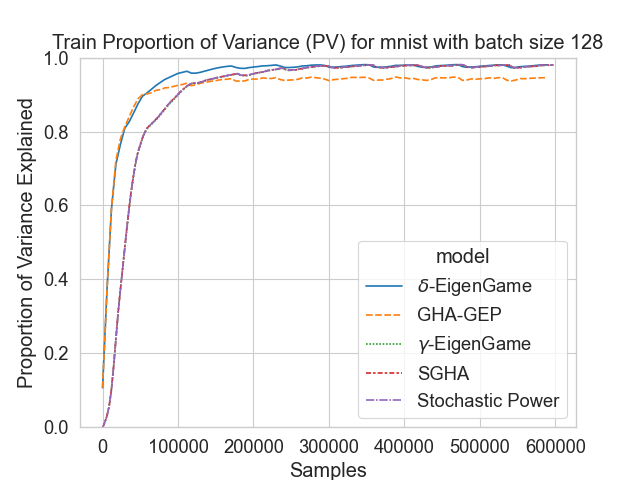}
         \label{fig:plsmnist}
     \end{subfigure}
     \hfill
     \begin{subfigure}[b]{0.32\textwidth}
         \centering
         \includegraphics[width=\textwidth]{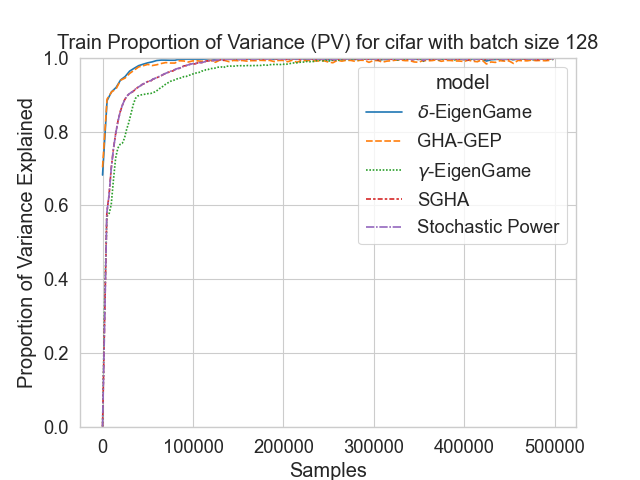}
         \label{fig:plscifar}
     \end{subfigure}
        \caption{PLS with stochastic mini-batches: proportion of variance captured with respect to Scipy ground truth by $\delta$-EigenGame vs prior work. The maximum value is 1.}
        \label{fig:plspv}
\end{figure}

\begin{figure}[H]
     \centering
     \begin{subfigure}[b]{0.31\textwidth}
         \centering
         \includegraphics[width=\textwidth]{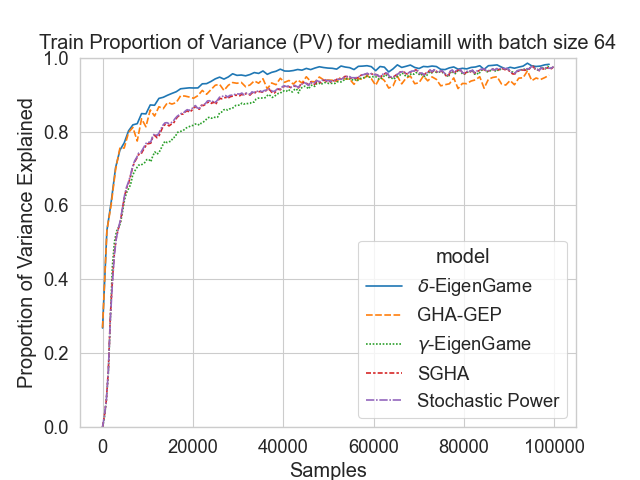}
         \label{fig:plsmediamill64}
     \end{subfigure}
     \hfill
     \begin{subfigure}[b]{0.31\textwidth}
         \centering
         \includegraphics[width=\textwidth]{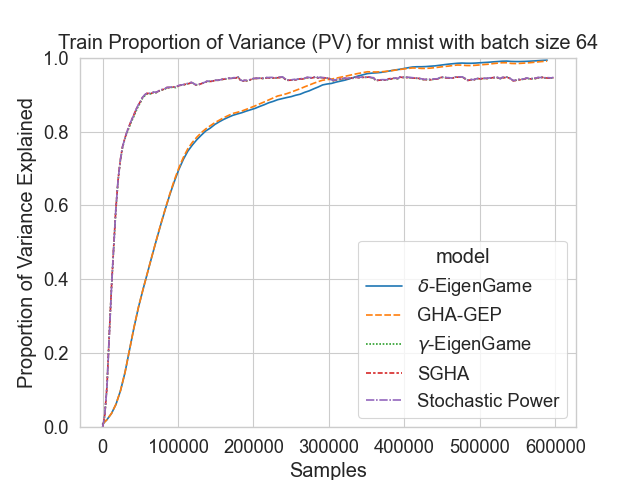}
         \label{fig:plsmnist64}
     \end{subfigure}
     \hfill
     \begin{subfigure}[b]{0.31\textwidth}
         \centering
         \includegraphics[width=\textwidth]{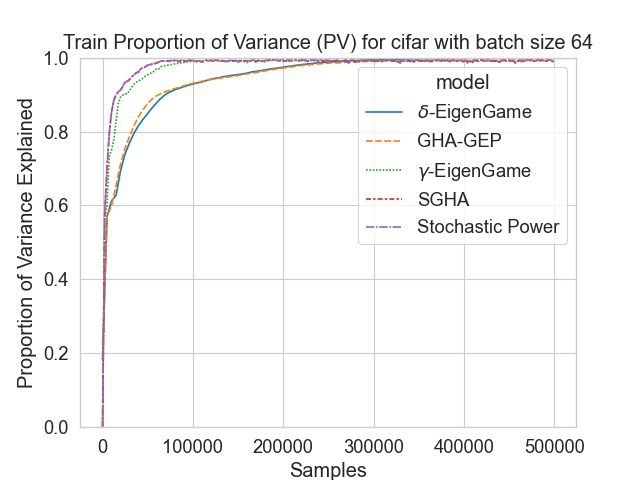}
         \label{fig:plscifar64}
     \end{subfigure}
     \newline
     \centering
     \begin{subfigure}[b]{0.31\textwidth}
         \centering
         \includegraphics[width=\textwidth]{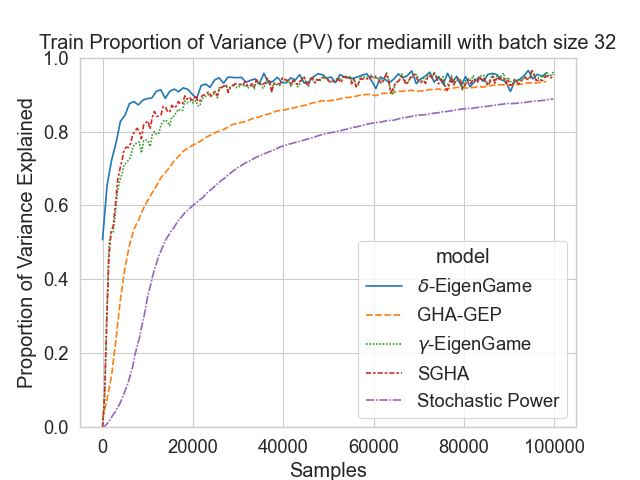}
         \label{fig:plsmediamill32}
     \end{subfigure}
     \hfill
     \begin{subfigure}[b]{0.31\textwidth}
         \centering
         \includegraphics[width=\textwidth]{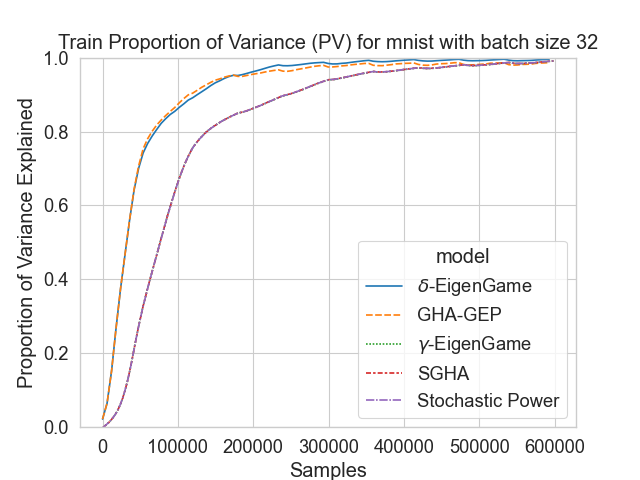}
         \label{fig:plsmnist32}
     \end{subfigure}
     \hfill
     \begin{subfigure}[b]{0.31\textwidth}
         \centering
         \includegraphics[width=\textwidth]{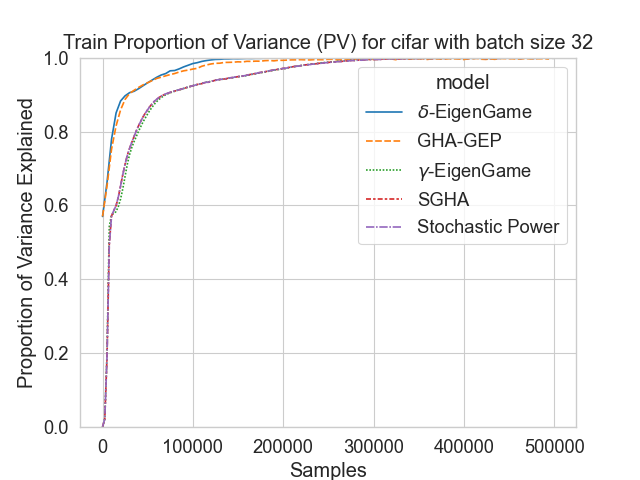}
         \label{fig:plscifar32}
     \end{subfigure}
        \caption{PLS with stochastic mini-batches of size 64 (top) and 32 (bottom): proportion of variance captured with respect to Scipy ground truth by $\delta$-EigenGame vs prior work. The maximum value is 1.}
        \label{fig:plspv32}
\end{figure}

\end{document}